\documentclass[journal]{IEEEtran}
\pdfoutput=1
\usepackage{graphicx}
\usepackage{amsmath,amssymb}
\usepackage{color}
\usepackage{flushend}
\usepackage[T1]{fontenc}
\usepackage{tablefootnote}
\usepackage{mathptmx}
\usepackage[mathscr]{euscript}
\usepackage[noconfigs,british]{babel} 
\usepackage{amsthm}
\usepackage{graphicx} 
\usepackage{cases}
\usepackage{multirow}
\usepackage{setspace}
\usepackage{empheq}
\usepackage[normalem]{ulem}

\usepackage{subfig}

\newtheorem{theorem}{Theorem}

\usepackage{booktabs}

\usepackage[ruled,linesnumbered]{algorithm2e}
\usepackage{threeparttable,booktabs}
\usepackage{etoolbox}
\usepackage{eqparbox}
\appto\TPTnoteSettings{\small}

\DeclareMathOperator{\atantwo}{atan2}

 \newtheorem{cor}{Corollary}[section]
 
 \newtheorem{prop}{Proposition}[section]
 \newtheorem{defn}{Definition}[section]

 \newtheorem{rem}{Remark}[section]

 \newcommand{\Real}{\mathbb{R}}

 \DeclareMathOperator{\sign}{sign}

\begin{document}

\title{
{\small{This paper has been accepted for publication by IEEE Transactions on Robotics}}\\ \vspace{-6pt}
\Huge{\hspace{6pt} Generalized Multi-Speed Dubins Motion Model} 
\thanks{This  work  was  supported  by  US  Office  of  Naval  Research  under  Award Number  N000141613032.  Any  opinions  or  findings  herein  are  those  of  the authors and do not necessarily reflect the views of the sponsoring agencies.}
\thanks{$^\dag$Dept. of Electrical and Computer Engineering, University of Connecticut, Storrs, CT 06269, USA.}
\thanks{$^\ddagger$Naval Undersea Warfare Center, Newport, RI 02841, USA}
\thanks{Digital Object Identifier 10.1109/TRO.2025.3554436}
\thanks{© 2025 IEEE.  Personal use of this material is permitted.  Permission from IEEE must be obtained for all other uses, in any current or future media, including reprinting/republishing this material for advertising or promotional purposes, creating new collective works, for resale or redistribution to servers or lists, or reuse of any copyrighted component of this work in other works.}
}

\author{ \begin{tabular}{ccc}
{James P. Wilson$^\dag$} & {Shalabh Gupta$^\dag$} & {Thomas A. Wettergren$^\ddagger$}\\
\end{tabular}
\\ \vspace{-12pt}
}

\maketitle
\begin{abstract}
The paper develops a novel motion model, called Generalized Multi-Speed Dubins Motion Model (GMDM), which extends the Dubins model by considering multiple speeds. While the Dubins model produces time-optimal paths under a constant-speed constraint, these paths could be suboptimal if this constraint is relaxed to include multiple speeds. This is because a constant speed results in a large minimum turning radius, thus producing paths with longer maneuvers and larger travel times. In contrast, multi-speed relaxation allows for slower speed sharp turns, thus producing more direct paths with shorter maneuvers and smaller travel times. Furthermore, the inability of the Dubins model to reduce speed could result in fast maneuvers near obstacles, thus producing paths with high collision risks. 

In this regard, GMDM provides the motion planners the ability to jointly optimize time and risk by allowing the change of speed along the path. GMDM is built upon the six Dubins path types considering the change of speed on path segments. It is theoretically established that GMDM provides full reachability of the configuration space for any speed selections. Furthermore, it is shown that the Dubins model is a  specific case of GMDM for constant speeds. The solutions of GMDM are analytical and suitable for real-time applications. The performance of GMDM in terms of solution quality (i.e., time/time-risk cost) and computation time is comparatively evaluated against the existing motion models in obstacle-free as well as obstacle-rich environments via extensive Monte Carlo simulations. The results show that in obstacle-free environments, GMDM produces near time-optimal paths with significantly lower travel times than the Dubins model while having similar computation times. In obstacle-rich environments, GMDM produces time-risk optimized paths with substantially lower collision risks. 
\end{abstract}


\thispagestyle{empty}

\vspace{-6pt}
\section{Introduction}
The paper develops a motion model, called Generalized Multi-Speed Dubins Motion Model (GMDM), for multi-speed vehicles. GMDM provides a higher fidelity than the Dubins model by allowing the change of speed along the path. The capability to reduce speed from its max value enables GMDM to 1) create sharp turns, thus resulting in overall shorter travel times than the Dubins paths, and 2) reduce collision risks around obstacles, thus resulting in safe maneuvering. It is theoretically established that GMDM achieves full reachability of the configuration space between any start and goal poses for any arbitrary set of speeds. It is also shown that GMDM reduces to the Dubins model for constant speeds. Furthermore, GMDM provides computationally efficient analytical solutions, thus making it suitable for real-time applications.

\begin{figure}[t!]
    \centering
    \subfloat{
    \centering
        \includegraphics[width=0.86\columnwidth]{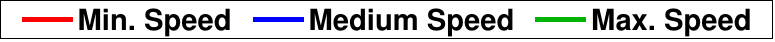}\label{fig:intro_legend}}\vspace{1pt}\\
        \setcounter{subfigure}{0}
    \centering
    \subfloat[Shorter-time path than Dubins in obstacle-free environments.]{
        \includegraphics[width=0.22\textwidth]{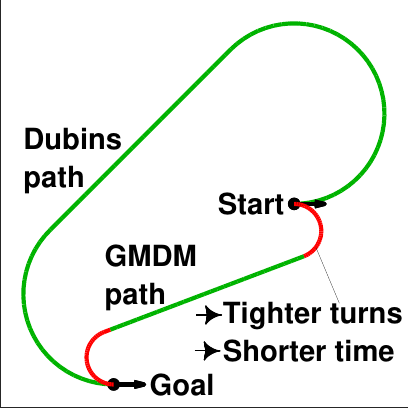}
        \label{fig:intro_a}
    }
    \subfloat[Safer path than Dubins in obstacle-rich environments.]{
        \includegraphics[width=0.22\textwidth]{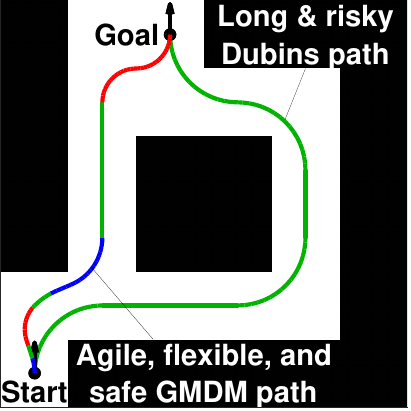}
        \label{fig:intro_b}
    }
    \caption{Comparison of GMDM with the Dubins paths.}  \vspace{-12pt}
    \label{fig:intro}
\end{figure}

\vspace{-6pt}
\subsection{Literature Review}

A fundamental problem in path planning is to find the minimum time path from a start pose to a goal pose while considering kinodynamic constraints of the vehicle. Dubins~\cite{D57}\cite{Shkel2001_DubinsDerivation} showed that in absence of obstacles, the shortest path for a curvature-constrained constant-speed vehicle between a pair of poses must be one of the following six canonical path types: LSL, RSR, LSR, RSL, LRL and RLR, where L(R) refers to a left (right) turn with the maximum curvature, and S indicates a straight line segment. Reeds-Shepp curves \cite{reeds1990optimal} extended Dubins curves by considering backward velocity. These models have analytical solutions that are easy to compute and implement. Further research considered field-of-view constraints \cite{Paolo2010}, environmental disturbances \cite{Mittal2020DubinsCurrents,McGee2007,Dolinskaya2012,Wolek2015DubinsCurrent}, multiple vehicles \cite{Loizou2008,CARE_Song2020}, obstacle-avoidance \cite{NPHard2}, 
the moving-target interception problem~\cite{BT13}\cite{meyer2015dubins}\cite{ding2019curvature}, the traveling salesman problem \cite{Savla2008_DubinsTSP,Bull02011_DubinsTSP,Ny2012_DubinsTSPDiscrete,Cohen2017_DubinsTSPProgram,Vana2020}, the orienteering problem \cite{Penicka2017,Tsiogkas2018}, and the coverage problem \cite{shen2019,song2018}. These problems, however, do not consider multiple speeds which are essential for time-optimal risk-aware motion planning \cite{Song2019_tstar}.

Recently, Wolek, et al. \cite{WCW16} developed a motion model for time-optimal planning in obstacle-free environments. Their solution is based on two speeds, i.e., the min and max speeds, which are sufficient for time-optimality in obstacle-free environments. However, the solutions of this model are not closed-form and require nonlinear solvers. Kucerov, et al. \cite{Kucerov2020_MultiRadiiDubins,Kucerov2021_MultiSpeedDubins} proposed multiple speeds and turning radii to find time-optimal paths for aircraft. However, Kucerov's model does not consider the LRL and RLR path types and requires higher computation times than the Dubins model without providing performance guarantees.  Finally, there exist other kinodynamic models that consider acceleration \cite{Donald1993_Kinodynamic} and curvature \cite{Boissonnat2003_Kinodynamic} constraints, but the high-dimensional nature of these problems make them infeasible for practical implementation \cite{Elbanhawi2014_SampleReview}.

Other motion models enforce curvature continuity (e.g., smooth transition from an L to an R segment) by considering a constraint on its derivative. Fermat's spiral has been used to ensure that transitions between Dubins segments are curvature-continuous \cite{Lekkas2013_FermatSpiral}. Fraichard and Scheuer \cite{Fraichard2004_CCSteer} extended Reeds-Shepp curves by using Euler's spiral, where the curvature changes linearly with respect to arc-length. Bruyninckx and Reynaerts \cite{Bruyninckx1997_PythagoreanCurves} enforced a continuously differentiable path by finding fifth-order Pythagorean hodographs that satisfy the continuity constraints. Qu et al. \cite{Qu2004_Spline} used piecewise-constant polynomials to construct complex paths that are twice-differentiable. Faigl and Vana \cite{Faigl2018} used Bézier curves to generate smooth paths to travel through a set of waypoints. 

\vspace{-12pt}
\subsection{Motivation}
The constant speed constraint in the Dubins model severely restricts the vehicle's maneuverability. In obstacle-free environments, the Dubins model produces suboptimal paths that are longer with larger travel times due to its inability to create sharp turns by reducing the speed. Fig.~\ref{fig:intro_a} shows an example where the GMDM path turns at the minimum speed to rapidly orient the vehicle towards the goal, thus producing a path which is both shorter and quicker than the Dubins path. 

The Wolek's motion model \cite{WCW16} is an improvement over the Dubins model as it produces time-optimal paths in obstacle-free environments. This is achieved by utilizing extremal (i.e., maximum and minimum) speeds, which are sufficient for time-optimality in obstacle-free environments. However, unlike Dubins, Wolek's solutions require numerical solvers for nonlinear equations, thus limiting their use in many real-time applications. 
On the other hand, GMDM provides analytical solutions that can be computed in real-time while approaching the time-optimal solutions of the Wolek's model. 

In obstacle-rich environments, risk evaluation becomes critical in motion planning~\cite{Pereira2013_SafteyOcean,Hernandez2016_Saftey2,Liu2017_Saftey3,Marchidan2022, SMART_Shen2023,CTCPP_Shen_2022}. Both time and risk costs are considered by the T$^\star$ algorithm \cite{Song2019_tstar} for time-optimal risk-aware planning. However, the solution quality of a high-level motion planner (e.g., RRT$^*$ and T$^\star$)  depends on the underlying motion model used to connect any two waypoints. Thus, a motion planner using the Dubins model lacks the capability to reduce risk and often produces long and risky paths. Moreover, due to constant speed, the Dubins paths lack flexible maneuvering around obstacles. Fig.~\ref{fig:intro_b} shows an example where the Dubins path goes through a long narrow corridor which is risky, while the GMDM path is shorter, quicker and safer by virtue of changing the speed. 

Similarly, a motion planner using the Wolek's model produces sub-optimal results in obstacle-rich environments because the two extremal speeds are not sufficient to minimize the time-risk cost. 
For instance, the straight line (S) segments in the Wolek's model always have the max speed for time-optimality; however, these segments should adopt lower speeds to reduce risk when approaching an obstacle. Similarly, the turn segments should adopt different slower speeds to tightly wrap around the obstacles of different geometries.

As such, the desired motion model should enable the selection of appropriate speeds for each path segment in order to produce time-risk-optimal paths that are both fast and safe. The speed selection should consider both travel time and risk based on the vehicle's orientation and distance from the obstacle. The model should preferably have analytical solutions for real-time computation. Finally, the model should be easy to understand and implement. To the best of our knowledge, no existing model in literature possesses all of these attributes. 

In this regard, this paper develops a motion model, GMDM, which is a generalization of the Dubins model that incorporates multiple speeds such that any path segment can select an appropriate speed. The model was first introduced in our prior work \cite{Wilson2019_GDubins}, where preliminary results showed that the model provides paths with reduced travel times and risks in obstacle-rich environments as compared to those obtained with existing motion models. This paper significantly improves upon our prior work in several ways as described below.

\vspace{-9pt}
\subsection{Contributions}
The main contributions of this paper are as follows:
\begin{itemize}
    \item Model development 
    \begin {itemize} 
    \item A fundamental extension of the Dubins motion model to GMDM considering multiple speeds for each path segment (i.e., L, S, and R) that enables both time and risk analysis in motion planning.
    \end{itemize}
    \item Model solution
       \begin{itemize}
    \item  Derivation of analytical solutions of the forward and inverse problems of the GMDM paths, which allow for real-time implementation.
    \item Theoretical guarantee of full reachability for any pair of start and goal poses for any selection of speeds.
    \item Analytical result to show that the Dubins model is a  specific case of GMDM for constant speeds. 
    \end{itemize}
\item Model validation 
\begin {itemize}
\item Comparative evaluation of GMDM against the Dubins and the Wolek's models in terms of solution quality (i.e., time/time-risk cost) and computation time in both obstacle-free and obstacle-rich environments.
\item Numerical results to show that in obstacle-free environments, the solution quality of GMDM approaches the time-optimal solutions of the Wolek’s model while enabling significantly faster computation.
\end{itemize}
\end{itemize}

\vspace{-12pt}  
\subsection{Organization}
    The rest of the paper is organized as follows. 
    { Section~\ref{sec:probform} formulates the time-risk optimal motion planning problem.} 
    Section~\ref{sec:DeriveMultispeedDubins} presents the details, solution, and properties of GMDM. Section~\ref{sec:MultispeedDubinsReachability} presents the reachability analysis of GMDM. Section~\ref{sec:results} presents the results in obstacle-free and obstacle-rich environments. 
    { Section~\ref{sec:limits} discusses the practical considerations of GMDM.}
    Section~\ref{sec:MultispeedDubinsConclusion} concludes the paper with recommendations for future work. Appendices \ref{AppendixA}, \ref{AppendixB}, \ref{AppendixC} derive the analytical results of Section~\ref{sec:DeriveMultispeedDubins} and \ref{sec:MultispeedDubinsReachability}.  

\section{Problem Description}\label{sec:probform}
{ 
\label{app:tstar}

Let $\mathscr{A}\subset \Real^2$ be a space possibly occupied with obstacles. Consider a vehicle maneuvering in this space whose motion is described as
\begin{subequations}
\label{eq:vehiclemotion}
\begin{align}
 \hspace{-0pt}   \dot{x}(t) &= v(t)\cos{\theta(t)}, \\
 \hspace{-0pt}   \dot{y}(t) &= v(t)\sin{\theta(t)}, \\
 \hspace{-0pt}   \dot{\theta}(t) &= \omega(t),  &&  
\end{align}
\end{subequations}
where $\textbf{p}(t)\triangleq(x(t),y(t),\theta(t)) \in \textup{SE}(2)$ is the vehicle pose at time $t$; $v(t)\in$ $\mathbb{V}=$ $[v_{min},v_{max}]$ is its speed in m/s at time $t$, where $v_{min},v_{max}\in \Real^+$; and $\omega(t)\in$ $\Omega =$$[-\omega_{max},\omega_{max}]$ is its angular speed (i.e., turning rate) in rad/s at time $t$, where $\omega_{max}\in \Real^+$, $\omega>0$ denotes a left turn, and conversely $\omega<0$ denotes a right turn.  
The curvature of the vehicle is defined as $\kappa(t)=|\omega(t)|/v(t)$, where $0\leq \kappa(t)\leq \omega_{max}/v_{min}$. The turning radius is $r(t)=1/\kappa(t)$,
where $\kappa(t)=0$ corresponds to forward movement on a straight line. 

Let $\mathbf{p}_{s}$ and $\mathbf{p}_{g}$ be the start and goal poses of the overall planning problem. The objective is to find the optimal control $\textbf{u}(t)\triangleq (v(t),\omega(t))\in \mathbb{U}$, where $\mathbb{U}= \mathbb{V} \times \Omega$, that satisfies the above constraints and drives the vehicle from $\mathbf{p}_{s}$ and $\mathbf{p}_{g}$, while avoiding obstacles and minimizing the time/time-risk cost. Let $\Gamma$ denote the set of all feasible paths from  $\mathbf{p}_{s}$ to $\mathbf{p}_{g}$.

The time-risk cost function \cite{Song2019_tstar} for evaluating any path $\gamma\in\Gamma$ is as follows. The risk cost is computed by estimating the collision time of the vehicle with any obstacles tangent to its current direction. Let $d_c(s)$ be the distance to the nearest obstacle in the direction of the pose at $\gamma(s)$. Let $v(s)$ be the speed of the vehicle at $s$. Thus, $t_c(s)\triangleq d_c(s)/v(s)$ is the collision time for the vehicle to hit the obstacle tangential to its current trajectory. The risk cost at $s$ is defined as:
\begin{equation}
    \mathscr{R}(s) =
    \begin{cases}
        1+\frac{t^\star}{t_c(s)}\log{\left(\frac{t^\star}{t_c(s)}\right)} & \text{if } t_c(s)\leq t^\star \\
        1 & \text{otherwise},
    \end{cases}
\end{equation}
where $t^\star$ is the risk-free collision threshold in seconds (i.e., the time for the vehicle to stop, maneuver around, or regain its control).  
The time cost is computed from the lengths and velocities of path segments. Thus, the joint cost function considering both risk and time is defined as:
\begin{equation}
\label{eq:tstarcostfun}
    J(\gamma) = \int_\gamma \big(\mathscr{R}(s)\big)^\lambda\cdot \frac{1}{v(s)} ds,
\end{equation}
where $\lambda\geq 0$ is a user-defined risk-weight. Thus,
the overall problem is to find the collision-free path $\gamma^*$ with minimal cost $J(\gamma^*)$ such that $J(\gamma^*) \leq J(\gamma)\; \forall \gamma\in\Gamma$. In this paper, we consider and evaluate both time-optimal ($\lambda = 0$) and time-risk optimal ($\lambda > 0$) motion planning problems. Typically, for planning a path in complex scenarios, a high-level planner is used with an underlying motion model (e.g., the Dubins model or GMDM).  

}

\section{GMDM}\label{sec:DeriveMultispeedDubins} 
This section presents the analytical details of GMDM. The Dubins model consists of a set of six constant speed path types: LSL, RSR, LSR, RSL, LRL and RLR. These can be solved for the duration spent on each path segment to obtain the minimum-time path between any two given poses in obstacle free environments. GMDM generalizes the Dubins model by relaxing the constant speed constraint. Specifically, GMDM allows the motion planner to select the speeds for L, S, and R segments of the Dubins path types.  Note that the different speeds for L and R segments lead to different turning radii, thus enhancing path maneuverability. On the other hand, reducing the speed on any segment near obstacles reduces the collision risk, thus enhancing path safety. This allows the motion planner to minimize the time and time-risk costs in obstacle-free and obstacle-rich environments, respectively.

\vspace{-6pt}

\vspace{-3pt}
\subsection{Motion Primitives}
Let $\textbf{u}\triangleq (v,\omega)\in \mathbb{U}$ be a constant input applied to the vehicle in (\ref{eq:vehiclemotion}) at pose $\textbf{p}(t)$ for a certain time duration $\tau \in \{\mathbb{R}^+\cup {0}\}$. Let $ \textup{M}:   \textup{SE}(2) \times  \mathbb{U} \times\{\mathbb{R}^{+}\cup{0}\} \rightarrow  \textup{SE}(2)$ be the motion primitive that describes the evolution of the pose $\textbf{p}(t)$ subject to the input $\textbf{u}$ for a time duration $\tau$. Thus, $\textbf{p}(t+\tau)= \textup{M}_{\textbf{u},\tau}(\textbf{p}(t))$. The motion primitive $\textup{M}$ is of two types as follows:  
\begin{equation}
    \label{eq:piecewiseElementaryMotion}
   \hspace{-0pt}      \textup{M}_{\textbf{u},\tau}(\textbf{p}(t)) = \begin{cases}
             \textup{C}_{\textbf{u},\tau}(\textbf{p}(t)) & \omega\neq 0\\
            \textup{S}_{\textbf{u},\tau}(\textbf{p}(t)) & \omega=0,
        \end{cases}
\end{equation}
where $ \textup{C}_{\textbf{u},\tau}(\cdot)$ and $ \textup{S}_{\textbf{u},\tau}(\cdot)$ denote the turning ($\omega\neq 0$) and straight line ($\omega=0$) motions, respectively.  
The turning motion $\textup{C}_{\textbf{u},\tau}(\cdot)$ is again of two types: left turn $\textup{L}_{\textbf{u},\tau}(\cdot)$ ($\omega > 0$) and right turn $ \textup{R}_{\textbf{u},\tau}(\cdot)$ ($\omega < 0$). 
Fig.~\ref{fig:PrimitivesGeometric} shows the motion primitives for straight, left, and right turn maneuvers, from where we can derive $\textbf{p}(t+\tau)$ geometrically as follows:

\vspace{6pt}
\begin{itemize}
\item For the turning motion: $\textbf{p}(t+\tau)  = 
     \textup{C}_{\textbf{u},\tau}(\textbf{p}(t))$, s.t.
\begin{subequations}\vspace{0pt}
\label{eq:CMotionPrimitive}
\begin{align}
 x(t+\tau)  & =  x(t)-\frac{v}{\omega}\Big(\sin{\theta(t)}-\sin{\big(\theta(t)+\omega\tau\big)}\Big),\\
 y(t+\tau)  &   =  y(t)+\frac{v}{\omega}\Big(\cos{\theta(t)}-\cos{\big(\theta(t)+\omega\tau\big)}\Big),\\
 \theta(t+\tau)   &   =  \theta(t)+\omega\tau.
   \end{align}
\end{subequations}
    \item For the straight line motion: $\textbf{p}(t+\tau)  =  \textup{S}_{\textbf{u},\tau}(\textbf{p}(t))$, s.t. 
    \begin{subequations}\vspace{0pt}
    \label{eq:SMotionPrimitive}
    \begin{align}    
 \hspace{-0pt} x(t+\tau) &  = x(t) +v\tau\cos{\theta(t)}, \\
\hspace{-0pt}  y(t+\tau)  & =  y(t) + v\tau\sin{\theta(t)}, \\
\hspace{-0pt} \theta(t+\tau) &  = \theta(t). &&
\end{align}
\end{subequations}
\end{itemize}

\begin{figure}[t!]
    \centering
    \subfloat[Straight line motion.]{
\includegraphics[width=0.1738\textwidth]{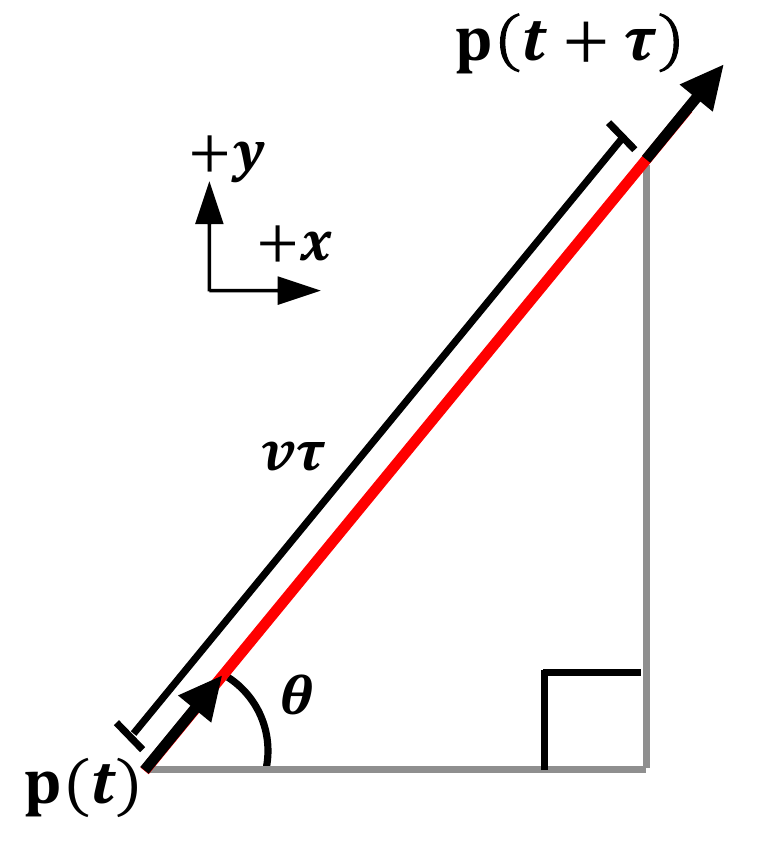}
        \label{fig:primitive_s}
    }\hspace{-10pt}
    \subfloat[Left turn motion.]{
        \includegraphics[width=0.1422\textwidth]{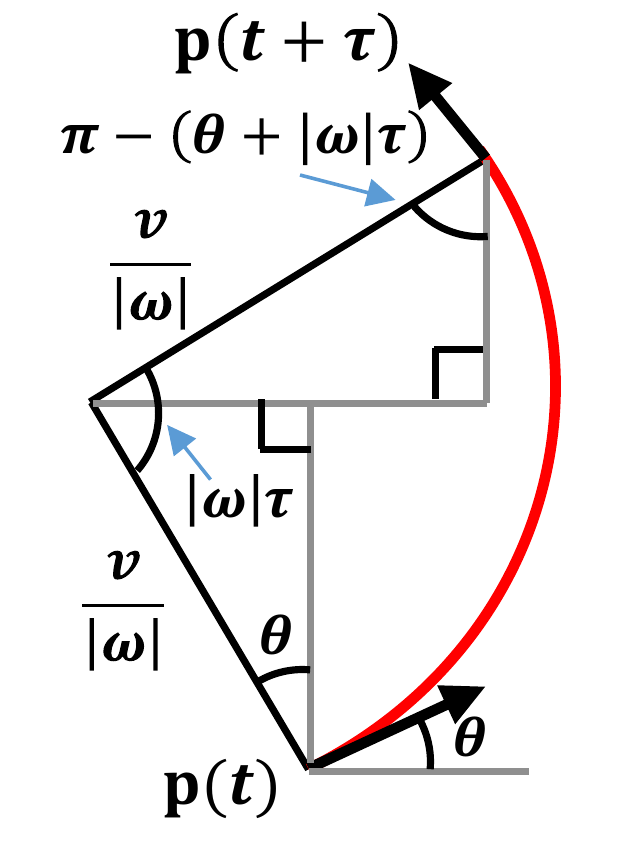}
        \label{fig:primitive_l}
    }\hspace{-10pt}
    \subfloat[Right turn motion.]{
        \includegraphics[width=0.1299\textwidth]{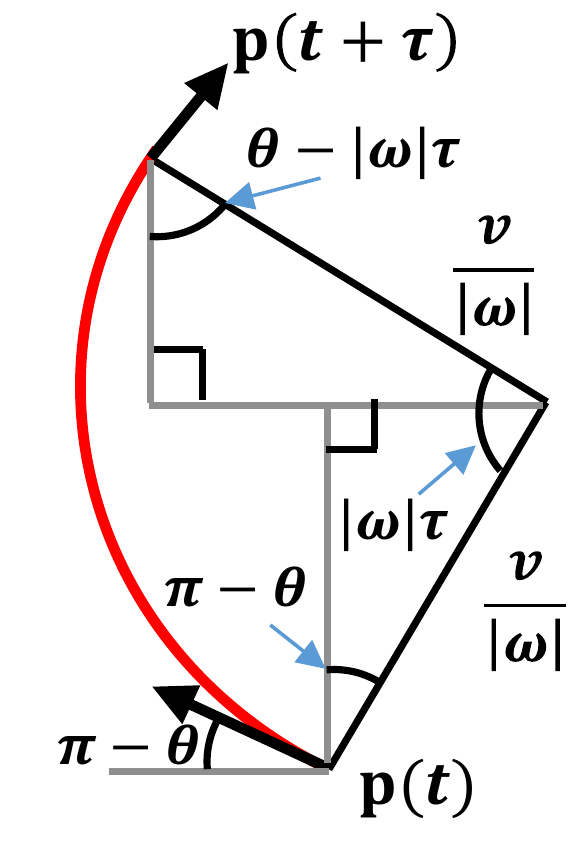}
        \label{fig:primitive_r}
    }
    \caption{Motion primitives of GMDM.} 
    \label{fig:PrimitivesGeometric}\vspace{-6pt}
\end{figure}

\begin{figure*}[t!]
    \centering
    \subfloat[CSC paths.]{
        \includegraphics[width=0.62\textwidth]{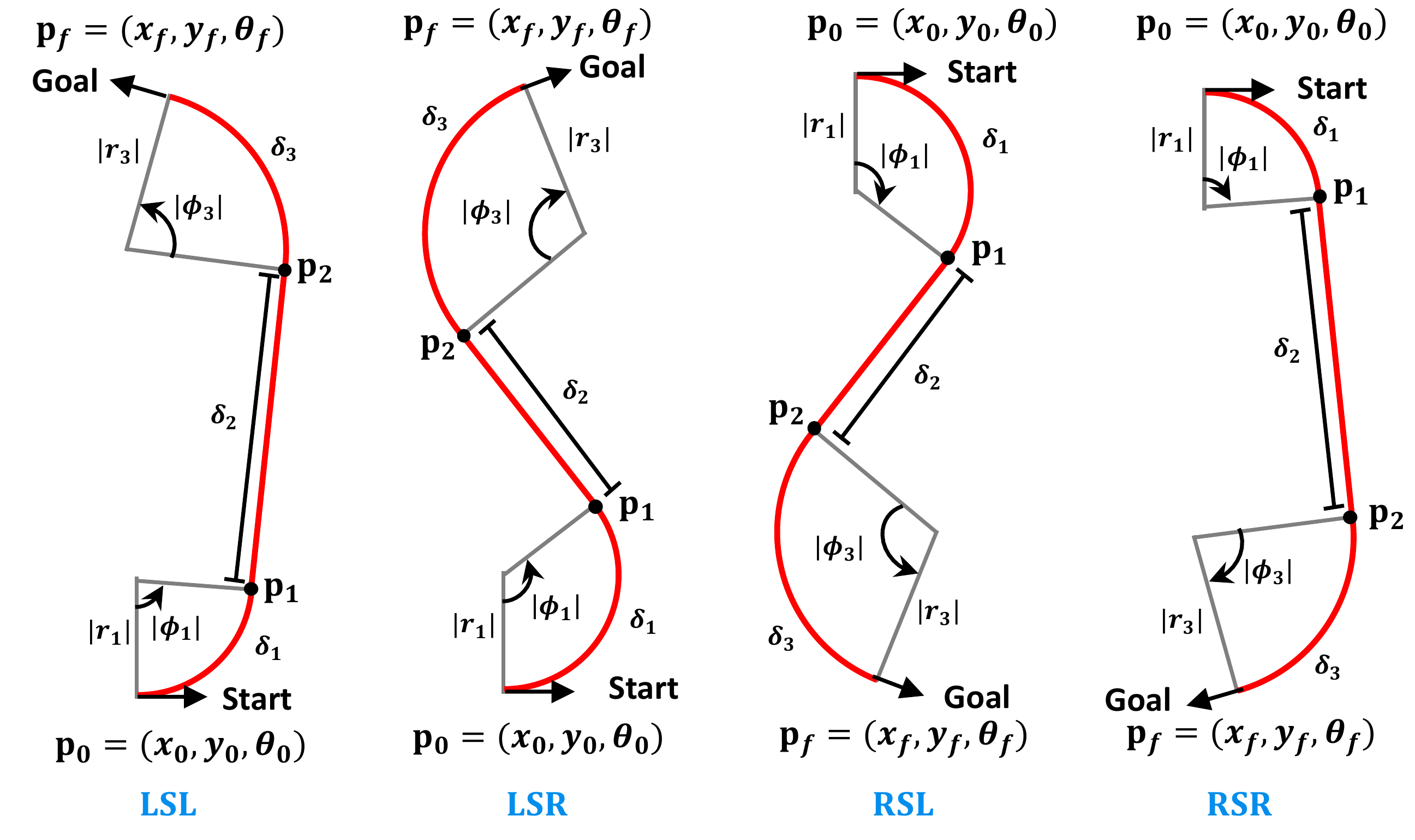}
        \label{fig:gmdmpath_csc}
    }
    \subfloat[CCC paths.]{
        \includegraphics[width=0.25\textwidth]{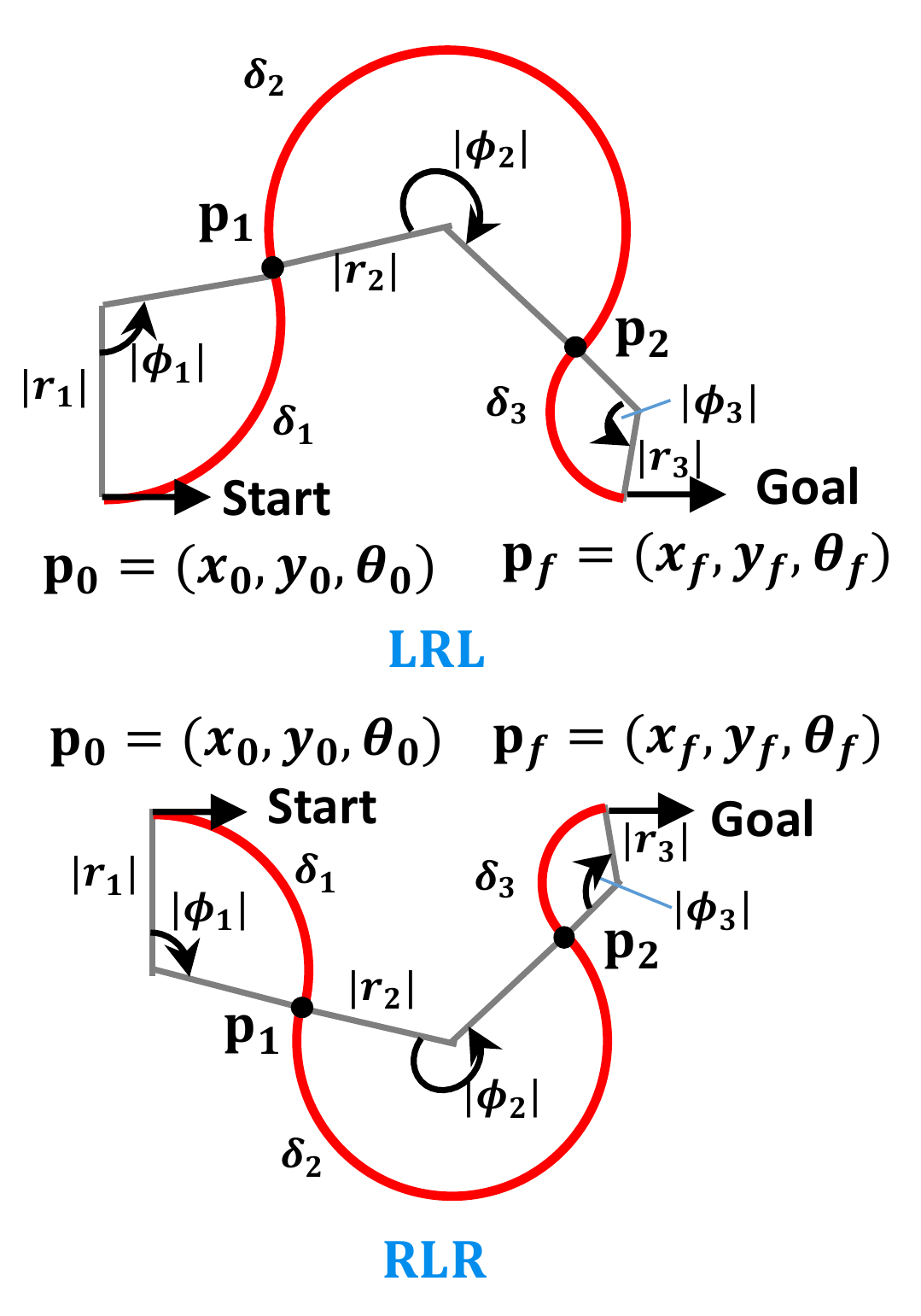}
        \label{fig:gmdmpath_ccc}
    }
    \caption{GMDM paths with different control inputs on each segment.} 
    \label{fig:gmdmpathtypes}\vspace{-12pt}
\end{figure*}

\vspace{-18pt}
\subsection{Motion Model}\vspace{-0pt}
Based on the motion primitives of (\ref{eq:CMotionPrimitive}) and (\ref{eq:SMotionPrimitive}), we now describe GMDM. Let $\textbf{p}_0=(x_0,y_0,\theta_0$) and $\textbf{p}_f=(x_f,y_f,\theta_f)$ denote the start and goal poses, respectively. The objective is to find the GMDM path between these poses. Similar to Dubins, a GMDM path connecting $\textbf{p}_0$ and $\textbf{p}_f$ consists of three path segments labeled by $i=1,2,3$. Let the start and end poses of segment $i$ be denoted by $\textbf{p}_{i-1}=(x_{i-1},y_{i-1},\theta_{i-1})$ and $\textbf{p}_{i}=(x_{i},y_{i},\theta_{i})$, respectively. Note that $\textbf{p}_f=\textbf{p}_3$. Let $\textbf{u}_i=(v_i,\omega_i)\in \mathbb{U}$ be the input to vehicle on segment $i$, applied for a time duration $\tau_i \in \{\mathbb{R}^+\cup {0}\}$. Then, $\textbf{p}_{i-1}$ evolves on segment $i$ to $\textbf{p}_{i}$ s.t. $\textbf{p}_{i}= \textup{M}_{\textbf{u}_i,\tau_i}(\textbf{p}_{i-1})$. Thus, each segment $i$ follows the motion primitive of either a turn $\textup{C}$ or a straight line $\textup{S}$.

This leads to two fundamental classes of GMDM paths: CSC (i.e., LSL, LSR, RSL and RSR) and CCC (i.e., LRL and RLR). For notation simplicity, we avoid the subscripts on the motion primitives and use them only  when needed. A CSC path first turns (either left or right), then goes straight, and finally turns (either left or right) before reaching the final pose. Similarly, a CCC path makes three turning maneuvers before reaching the final pose, where $\sign\omega_1\neq \sign\omega_2$ and $\sign\omega_2\neq \sign\omega_3$ (i.e., consecutive turning motions must be in different directions). Note: the other complex path types (e.g., paths with more than three segments) 
are not included in GMDM and are beyond the scope of this work.

Furthermore, while GMDM is built on the six Dubins path types, it allows a different speed on each path segment. Thus, the total number of GMDM path types depend on the number speeds allowed for each path segment. We show later that GMDM provides full reachability of the configuration space for any speed selections. Figure~\ref{fig:gmdmpathtypes} shows the GMDM paths with a different control input on each segment.

\vspace{-5pt}
\subsection{Model Analysis}\vspace{-0pt}
For GMDM analysis, we discuss the forward and inverse problems. First, we define the following parameters.
\begin{subequations}
    \label{eq:param1}
    \begin{align}
    \begin{split}
    r_{i}\triangleq\frac{v_i}{\omega_i}, \  i=1,2,3,
    \end{split}\\
\begin{split}
    r_{ij}\triangleq r_i-r_j, \  i,j=1,2,3,
    \end{split}\\
    \begin{split} 
    \delta_i \triangleq v_i\tau_i, \  i=1,2,3,
    \end{split}\\
    \begin{split}
    \phi_i \triangleq \omega_i\tau_i, \  i=1,2,3,
    \end{split}\\
    \begin{split}
        \theta_{ij}\triangleq \bmod(\theta_i-\theta_j,2\pi\sign{(\omega_i)}), \ i,j=0,1,2,3,
    \end{split}
    \end{align}
   \end{subequations} 
{where $\bmod(a,m)\triangleq a-m\lfloor \frac{a}{m}\rfloor \ \forall a,m\in\Real$ \cite{knuth1997_mod}.} Note that $|r_i|$, $|\phi_i|$ and $\delta_i$, $i=1,2,3$,  represent the turning radius, rotation and length of a path segment $i$, respectively. The rotations satisfy the constraint $|\phi_i|< 2\pi$.

\vspace{6pt}
\subsubsection{The forward problem analysis}
\begin{defn}[Forward problem]
The forward problem aims to find the final pose $\textbf{p}_f$ of a GMDM path given the start pose $\textbf{p}_0$, the control inputs $\textbf{u}_i$ and time durations $\tau_i$ of its segments $i=1,2,3$. 
\end{defn}
The final pose of a GMDM path is obtained by applying the motion primitive for each segment consecutively as follows: 
\begin{equation}\label{GMDMposeevolution}
\textbf{p}_f= \textup{M}_{u_3,\tau_3}( \textup{M}_{u_2,\tau_2}( \textup{M}_{u_1,\tau_1}(\textbf{p}_0))).
\end{equation}

\vspace{3pt}
\begin{prop}[CSC forward]\label{propcscforward}
Given $\textbf{p}_0$ and $(\textbf{u}_i,\tau_i)$, $i=1,2,3$, the final pose $\textbf{p}_f$ of a CSC path is given as
\begin{subequations}\label{eqcscforward}
    \begin{align}
    \begin{split}\label{propxfcsc}
        x_f ={}& x_{0}-r_{1}\sin\theta_0-r_{31}\sin(\theta_0+\phi_1) + \\
        & \delta_2\cos(\theta_0+\phi_1) + r_{3}\sin(\theta_0+\phi_1+\phi_3),
    \end{split}\\
    \begin{split}\label{propyfcsc}
        y_f ={}&  y_{0}+r_{1}\cos\theta_0+ r_{31}\cos(\theta_0+\phi_1) + \\
        & \delta_{2}\sin(\theta_0+\phi_1) -r_{3}\cos(\theta_0+\phi_1+\phi_3),
    \end{split}\\
    \begin{split}\label{propthetafcsc}
        \theta_f ={}& \bmod(\theta_0+\phi_1+\phi_3,2\pi).
    \end{split}
    \end{align}
    \end{subequations}
\end{prop}
\begin{proof} See Appendix \ref{proofpropcscforward}.
\end{proof}

\begin{prop}[CCC forward]\label{propcccforward}
Given $\textbf{p}_0$ and $(\textbf{u}_i,\tau_i)$, $i=1,2,3$, the final pose $\textbf{p}_f$ of a  CCC path is given as  
\begin{subequations}\label{eqcccforward}
    \begin{align}
    \begin{split} \label{propxfccc}
        x_f ={}& x_0-r_{1}\sin\theta_0+r_{12}\sin(\theta_0+\phi_1)
        + \\ & r_{23}\sin(\theta_0+\phi_1+\phi_2) 
         +r_{3}\sin{(\theta_0+\phi_1+\phi_2+\phi_3)},
    \end{split}\\
    \begin{split} \label{propyfccc}
        y_f ={}&  y_0+r_{1}\cos\theta_0-r_{12}\cos(\theta_0+\phi_1)- \\ &
r_{23}\cos(\theta_0+\phi_1+\phi_2) - r_{3}\cos{(\theta_0+\phi_1+\phi_2+\phi_3)},
    \end{split}\\
    \begin{split} \label{propthetafccc}
        \theta_f ={}&  \bmod(\theta_0+\phi_1+\phi_2+\phi_3,2\pi).
    \end{split}
    \end{align}
    \end{subequations}
\end{prop}
\begin{proof} See Appendix \ref{proofpropcccforward}.
\end{proof}

\begin{figure*}
\centering
 \subfloat[CSC paths.]{
   \hspace{0pt} \includegraphics[width=0.60\textwidth]{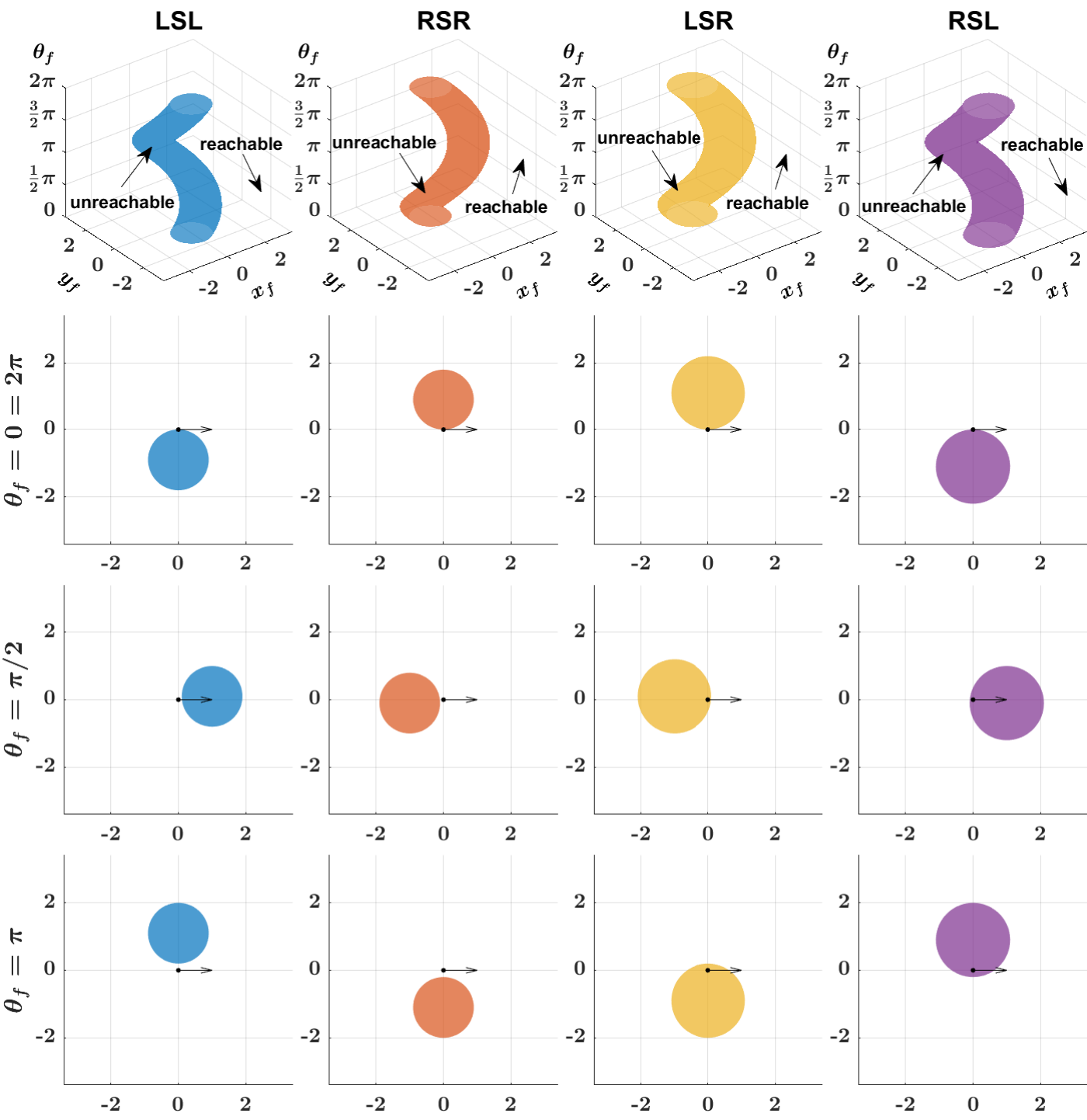}
        \label{fig:ReachCSC} 
    }\hspace{0pt}
 \subfloat[CCC paths.]{
    \includegraphics[width=0.30\textwidth]{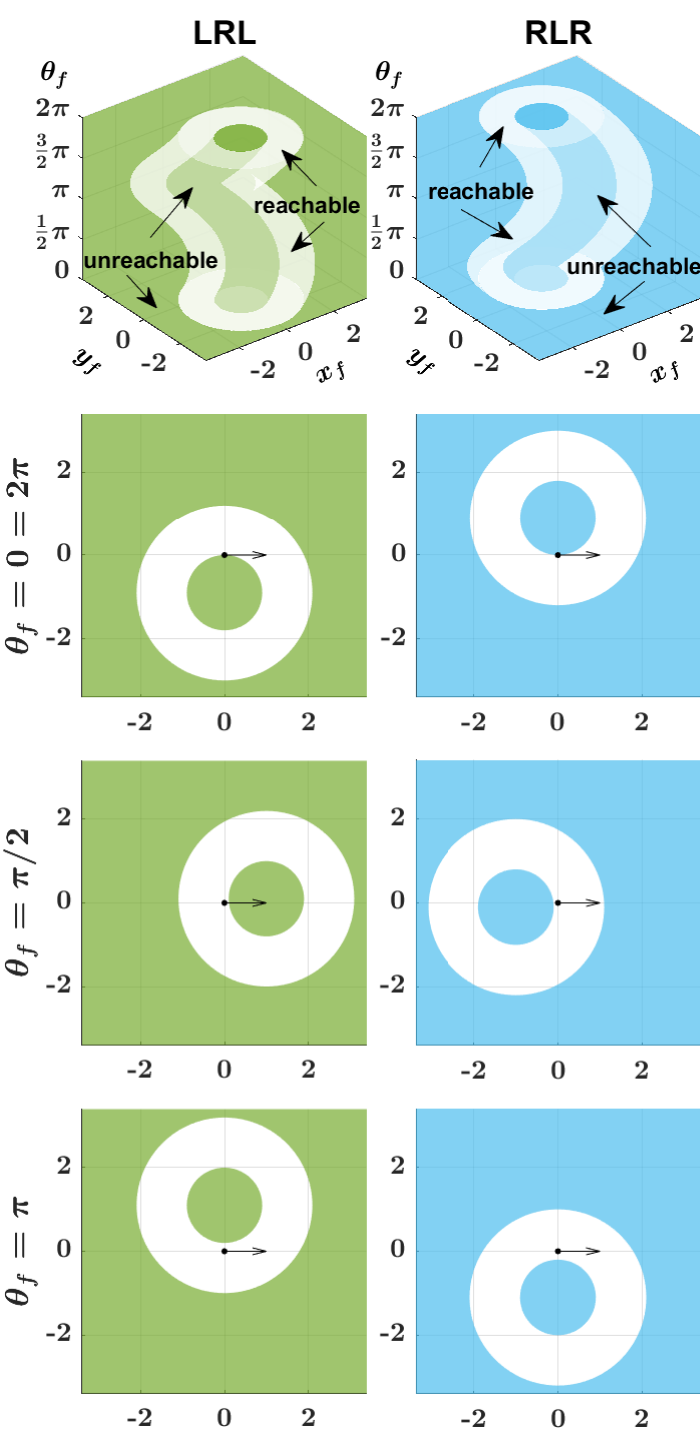}
    \label{fig:ReachCCC} 
}
\caption{Visualization of the reachable (white) and unreachable (colored) sets from $\textbf{p}_0=(0,0,0)$ for the (a) CSC and (b) CCC paths. The top row shows the reachability in the SE(2) space. The bottom three rows show the cross-sections of $x_f$-$y_f$ planes for different $\theta_f$. The plots are drawn for the control inputs $(v_1,v_2,v_3)=(0.1,0.5,1)$ m/s and $|\omega_i|\in\{0,1\}$ rad/s, which correspond to $(|r_1|,|r_3|)=(0.1,1.0)$ m for all paths and $|r_2|=0.5$ m for the CCC paths.}
\label{fig:ReachGMDM}\vspace{-6pt}
\end{figure*}

\vspace{3pt}
\subsubsection{The inverse problem analysis}
\begin{defn}[Inverse problem]
The inverse problem aims to find the time durations $\tau_i$ of GMDM path segments $i=1,2,3$, given the start and goal poses ($\textbf{p}_0$, $\textbf{p}_f$) and the control inputs $\textbf{u}_i$. The total travel time of the GMDM path is given as $\mathscr{T}=\sum^3_{i=1}\tau_i$.
\end{defn}

Let us define the following parameters:
\begin{subequations}
    \label{eq:AB}
    \begin{align}
    \begin{split}
    a \triangleq x_f-x_0+r_{1}\sin\theta_0-r_{3}\sin\theta_f,
    \end{split}\\
    \begin{split}
    b \triangleq y_f-y_0-r_{1}\cos\theta_0+ r_{3}\cos\theta_f,
    \end{split}
    \end{align}
   \end{subequations}  
which are known for the inverse problem.

\vspace{0pt}
\begin{prop}[CSC inverse]\label{propcscinverse}
Given $\textbf{p}_0$, $\textbf{p}_f$ and $\textbf{u}_i$, the time durations  $\tau_i$, $i=1,2,3$, of CSC path segments are given as
     \begin{equation} \label{eq:CSCsolns}
         \tau_1= \frac{\theta_{10}}{\omega_1},
         \tau_2= \frac{\delta_2}{v_2}, \ \textrm {and} \
          \tau_3= \frac{\theta_{31}}{\omega_3},
     \end{equation}
 where
    \begin{equation}
        \label{eq:CSCsolntheta1prop}
        \theta_1=\arcsin\Big(\frac{-r_{31}}{\sqrt{a^2+b^2}}\Big)-\atantwo(-b,a),
\end{equation} 
and 
\begin{equation} \label{eq:CSCsolndelta2prop}
        \delta_2 = \sqrt{a^2+b^2-r^2_{31}}.
    \end{equation}
\begin{proof} See Appendix \ref{proofpropcscinverse}.
\end{proof}      
\end{prop}

\vspace{-6pt}
\begin{prop}[CCC inverse]\label{propcccinverse}
Given $\textbf{p}_0$, $\textbf{p}_f$ and $\textbf{u}_i$, the time durations  $\tau_i$, $i=1,2,3$, of CCC path segments are given as
     \begin{equation}\label{eq:CCCsolns}
         \tau_1= \frac{\theta_{10}} {\omega_1}, 
         \tau_2=\frac{\theta_{21}}{\omega_2}, 
         \tau_3= \frac{\theta_{32}}{\omega_3}, 
     \end{equation}
where 
\begin{equation} \label{eq:CCCsolntheta1prop}
        \theta_1=\pi-\arcsin\Big(\frac{a^2+b^2+r_{12}^2-r_{23}^2}{2r_{12}\sqrt{a^2+b^2}}\Big)-\atantwo(-b,a),
    \end{equation}
and
\begin{equation}\label{eq:CCCsolntheta2prop}
        \theta_2=\pi-\arcsin\Big(\frac{a^2+b^2+r_{23}^2-r_{12}^2}{2r_{23}\sqrt{a^2+b^2}}\Big)-\atantwo(-b,a).
    \end{equation}
\end{prop}
\begin{proof} See Appendix \ref{proofpropcccinverse}.
\end{proof}

\begin{figure*}[t!]
    \centering
    \subfloat{
    \centering
        \includegraphics[width=0.32\textwidth]{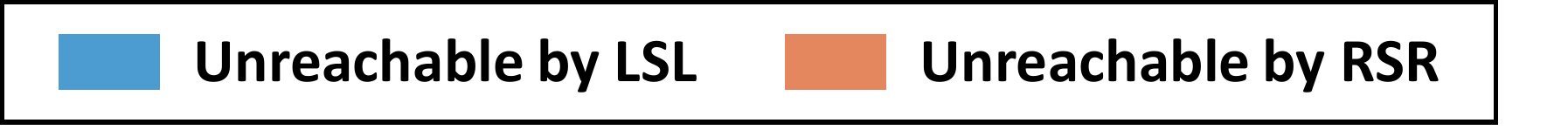}
    }
    \vspace{-5pt}\\
    \setcounter{subfigure}{0}
    \centering
    \subfloat[Visualization in SE(2) space.]{
        \includegraphics[width=0.3\textwidth]{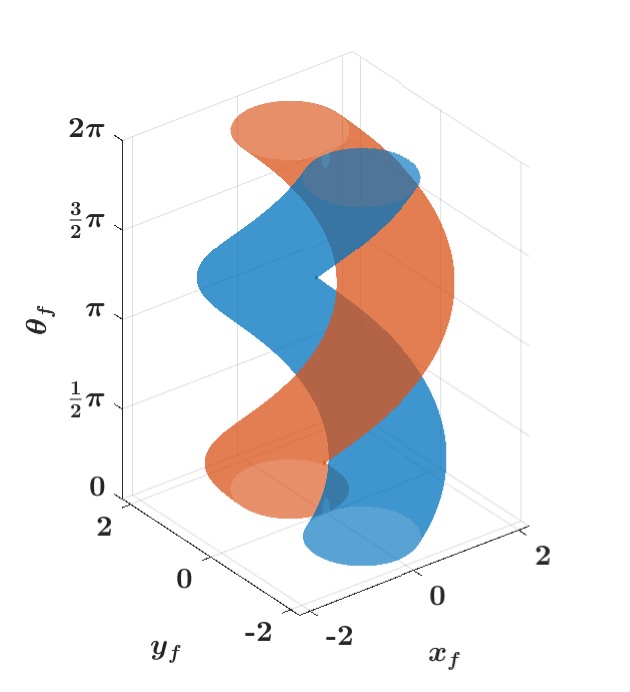}
        \label{fig:reachthma}
    }
    \subfloat[Visualization of the $x_f$-$y_f$ planes for different $\theta_f$.]{
        \includegraphics[width=0.6\textwidth]{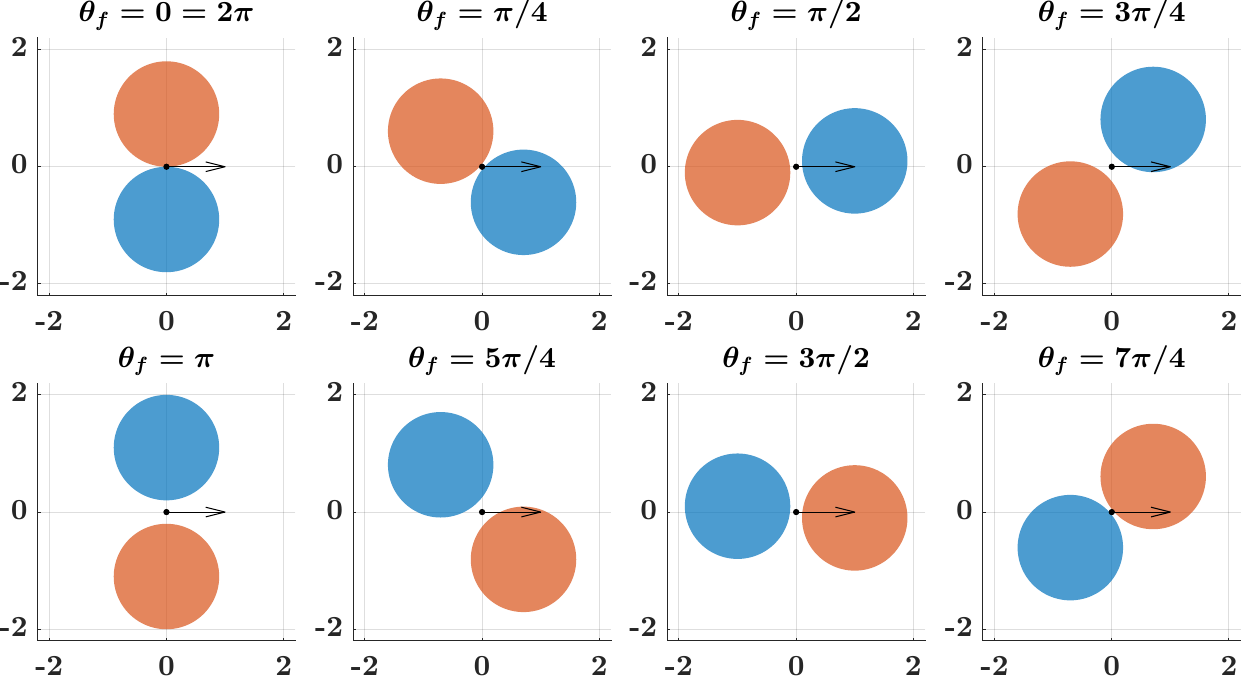}
        \label{fig:reachthmb}
    }
    \caption{Visualization of the unreachable regions of LSL and RSR path types. Since these regions are disjoint, full reachability is achieved by GMDM as per   Theorem~\ref{Th:GMDM_full_reachability}. The plots are drawn for $\textbf{p}_0=(0,0,0)$ and $(|r_1|, |r_3|) = (0.1,1.0)$. }
    \label{fig:ReachTheorem}\vspace{-6pt}
\end{figure*}

\begin{cor}
The  \textup{GMDM} path types reduce to the Dubins set \cite{Shkel2001_DubinsDerivation} when 
\begin{itemize}
\item [1.] $|r_1|=|r_3|$ for CSC.
\item [2.] $|r_1|=|r_2|=|r_3|$ for CCC.
\end{itemize}
\end{cor} 
\begin{proof}
The Dubins set \cite{Shkel2001_DubinsDerivation} follows by plugging conditions 1 and 2 in (\ref{eqcscforward}) and (\ref{eqcccforward}), respectively.
\end{proof}

\subsection{Computation of Optimal GMDM Path}

The above description first introduced GMDM and then provided the solutions for the forward and inverse problems for a GMDM path given the control inputs $\textbf{u}_i$ for path segments. The optimal control inputs $\textbf{u}_i^*$ for producing a GMDM path are selected by a higher-level planner. For practical implementation, the optimal control inputs are determined as follows. Given the start and goal poses, first a discrete set of control inputs are used to generate different GMDM path types. 
For example, consider $v_i\in\{v_{min},v_{max}\}$ and $|\omega_i|\in\{0, \omega_{max}\}$, then the $6$ Dubins path types generate a total of $6\times2^3\times1 =48$ different GMDM path types. Then, the inverse solution for each of these path types is determined, and its path quality (i.e., time/time-risk \cite{Song2019_tstar} cost) is evaluated. Due to the closed-form nature of the solutions, they can be computed quickly in real time. Finally, the control $\textbf{u}_i^*$ that provides the best path quality is selected to yield the GMDM path. Section~\ref{sec:results} shows several examples for different applications.

\vspace{-6pt}
\section{Reachability Analysis}
\label{sec:MultispeedDubinsReachability} \vspace{-0pt}
This section presents the reachability analysis of GMDM.

\begin{defn}[Reachability] A pose $\textbf{p}_f$ is said to be reachable { from pose $\textbf{p}_0$} if there exists a GMDM path from $\textbf{p}_0$ to $\textbf{p}_f$. 
\end{defn}

Let $\mathcal{R}_{j}\subseteq \text{SE(2)}$ denote the reachable set of a GMDM path type $j$, $j$ = CSC or CCC, that includes all final poses $\textbf{p}_f$ that are reachable from the start pose $\textbf{p}_0$. Let
\begin{subequations}
    \label{eq:CD}
\begin{align}
c \triangleq x_0-r_{1}\sin\theta_0+r_{3}\sin\theta_f,\\
d \triangleq y_0+r_{1}\cos\theta_0-r_{3}\cos\theta_f.
\end{align}

\end{subequations}

\begin{theorem}[CSC Reachability]\label{Th:GMDM_CSC_reachability}
The reachable set $\mathcal{R}_{CSC}$ of GMDM satisfies the following
\begin{equation}
       \label{eq:Multispeed_CSC_Reach_Solution}
        \begin{split}
         (x_f-c)^2 + (y_f-d)^2  \geq &  r_{31}^2.   
        \end{split}
    \end{equation}
\end{theorem}
\begin{proof}
See Appendix \ref{prf:GMDM_CSC_reachability}.
\end{proof}
Fig. \ref{fig:ReachCSC} visualizes the reachable set $\mathcal{R}_{CSC}$ derived from (\ref {eq:Multispeed_CSC_Reach_Solution}). The top row of Fig. \ref{fig:ReachCSC} shows $\mathcal{R}_{CSC}$ in SE(2) space with four reachability plots corresponding to LSL, RSR, LSR and RSL path types. As seen in each of these plots, the reachable region of a path type lies outside a tube including its boundary while the region inside this tube is unreachable. The bottom three rows of Fig. \ref{fig:ReachCSC} show the cross sections of $x_f$-$y_f$ planes for different $\theta_f$. As seen from each of these cross-sections, the reachable region of a path type lies outside an open circle with center $(c,d)$ and radius $|r_{31}|$.

\begin{rem} The control inputs $\textbf{u}_i$, $i=1,2,3$, affect the radii $|r_i|$, which in turn affect the center ($c,d$) and radius $|r_{31}|$ of the circular region  in (\ref{eq:Multispeed_CSC_Reach_Solution}). Hence, the reachability of a CSC path type depends on the choice of controls. 
\end{rem}

\begin{cor}
The Dubins LSL and RSR path types each provide full reachability of the SE(2) space.
\end{cor}
\begin{proof}
    For the Dubins LSL and RSR path types, $r_1=r_3$, thus $r_{31}=0$. From (\ref{eq:Multispeed_CSC_Reach_Solution}) we get $(x_f-c)^2 + (y_f-d)^2  \geq 0$, which is true for all $p_f=(x_f,y_f,\theta_f)\in SE(2)$. 
\end{proof}

\begin{theorem}[CCC Reachability]\label{Th:GMDM_CCC_reachability}
The reachable set $\mathcal{R}_{CCC}$ of GMDM satisfies the following:
    \begin{equation}
        \label{eq:Multispeed_CCC_Reach_Solution}
        \begin{split}
         r_{31}^2 \leq & \big(x_f-c\big)^2 + \big(y_f-d\big)^2
        \leq \big(r_{12} - r_{23}\big)^2
        \end{split}
    \end{equation}
\end{theorem}
\begin{proof}
See Appendix \ref{prf:GMDM_CCC_reachability}.
\end{proof}

Fig. \ref{fig:ReachCCC} visualizes the reachable set $\mathcal{R}_{CCC}$ derived from (\ref {eq:Multispeed_CCC_Reach_Solution}). The top row of Fig. \ref{fig:ReachCCC} shows $\mathcal{R}_{CCC}$ in SE(2) space with two reachability plots corresponding to LRL and RLR path types. As seen in each of these plots, the reachable region for a path type lies inside an annular region within a tube with inner and outer boundaries included, while everything else is unreachable. The bottom three rows of Fig. \ref{fig:ReachCCC} show the cross sections of $x_f$-$y_f$ planes for different $\theta_f$. As seen from each of these cross-sections, the reachable region of a path type lies inside an annulus of a closed circle with center $(c,d)$, inner radius $|r_{31}|$ and outer radius $|r_{12}-r_{23}|$.

\begin{rem}
The control inputs $\textbf{u}_i$, $i=1,2,3$, affect the radii $|r_i|$, which in turn affect the center ($c,d$)  and radii $|r_{31}|$ and $|r_{12} - r_{23}|$ of the annulus region in (\ref{eq:Multispeed_CCC_Reach_Solution}). Hence, the reachability of a CCC path type depends on the choice of controls. 
\end{rem}

\begin{cor}
The Dubins $LRL$ and $RLR$ paths each provide reachability inside a circle with center $(c,d)$ and radius $4|r_1|$.
\end{cor}
\begin{proof}
For the Dubins LRL and RLR path types, $r_1=-r_2=r_3$, thus $r_{31}=0$ and $r_{12}-r_{23}=r_1-r_2-r_2+r_3=4r_1$. From ~(\ref{eq:Multispeed_CCC_Reach_Solution}) we get   
 $\big(x_f-c\big)^2 + \big(y_f-d\big)^2
        \leq (4r_1)^2$. 
\end{proof}

The reachable sets for the CSC and CCC path types in Theorems~\ref{Th:GMDM_CSC_reachability} and \ref{Th:GMDM_CCC_reachability} above were numerically verified by solving a large number of goal poses. 
From the results of  Theorems~\ref{Th:GMDM_CSC_reachability} and \ref{Th:GMDM_CCC_reachability}, it is clear that any individual GMDM path type from either the CSC or CCC class does not provide full reachability. However, the next theorem guarantees full  reachability of GMDM using a union of the LSL and RSR path types.

\begin{theorem}[Full Reachability]\label{Th:GMDM_full_reachability}
GMDM achieves full reachability of the SE(2) space using LSL and RSR path types.
\end{theorem}
\begin{proof}
See Appendix \ref{prf:GMDM_full_reachability}.
\end{proof}

Fig.~\ref{fig:ReachTheorem} visualizes and numerically validates Theorem \ref{Th:GMDM_full_reachability}. Fig. \ref{fig:reachthma} shows the unreachable regions of LSL and RSR path types in the SE(2) space, while Fig. \ref{fig:reachthmb} shows these regions in the cross-sections of the $x_f-y_f$ plane for different $\theta_f$. The plots are drawn for $\textbf{p}_0=(0,0,0)$ and $(r_1,r_3)=(0.1,1.0)$. Clearly, the unreachable regions of LSL and RSR path types are disjoint,  thus verifying that all poses are reachable when using at least these two path types. Numerical evaluations for other sets of turning radii show similar results.

\begin{rem}
The pair of $LSL$ and $RSR$ path types provides full reachability to GMDM. However, it can be easily shown that for any other pair, the unreachable regions of path types are not always disjoint, thus not providing full reachability.
\end{rem}

\vspace{-6pt}
\section{Results and Discussion}
\label{sec:results}
This section presents the comparative evaluation results of GMDM with the baseline models for time-optimal and time-risk optimal planning in different scenarios. The performance is measured by the solution quality (i.e., time/time-risk cost), computation time and collision risk of the produced path. 

The solutions of the Dubins model and GMDM are closed form and obtained analytically, whereas the solutions of the Wolek model required the IPOPT solver \cite{IPOPT}. All motion models are coded in C++. The computations were done on an Intel Core-i7 $7700$ processor with 32GB of RAM on Ubuntu $16.04$ LTS. A curvature-constrained vehicle with $v_{\min} = 0.3$ m/s, $v_{\max}=1.0$ m/s and $\omega_{\max} = 1.0$ rad/s is considered, with the associated turning radii of $r_{\min}=0.3$ m and $r_{\max}=1.0$ m. The results are generated by extensive Monte Carlo simulations and a detailed discussion on the advantages of GMDM for time/time-risk optimal planning is presented.

\vspace{0pt}

\subsection{Discussion of the Baseline Models and GMDM}

\subsubsection{The Dubins model}
This is the first baseline model for comparison. It produces time-optimal paths for constant-speed curvature-constrained vehicles. Let $\Gamma_D$ denote the set of Dubins path types: LSL, LSR, RSL, RSR, LRL, RLR. As discussed earlier, the Dubins model might produce suboptimal paths for multi-speed vehicles due to its inability to create sharp turns; however, it is the simplest model with closed form solution, thus suitable for onboard real-time implementation.

\begin{table}[t!]
\centering
\small
\caption{Sets of candidate path types.}
\label{tab:wolek34}
\begin{tabular}{ll|lll}
\multicolumn{2}{c}{$\Gamma_W$ (Wolek, et al. \cite{WCW16})} & \multicolumn{2}{c}{$\Gamma'_{G2}$ (GMDM$'$-2) \tnote{1}} \\ \bottomrule
No. & Path Type            & No.       & Path Type \\ \bottomrule
1   & L$^+$S$^+$L$^+$          & 1         & L$^+$S$^+$L$^+$ \\
2   & L$^+$S$^+$R$^+$          & 2         & L$^+$S$^+$R$^+$ \\
3   & R$^+$S$^+$L$^+$        &  3         & R$^+$S$^+$L$^+$ \\
4   & R$^+$S$^+$R$^+$          &  4         & R$^+$S$^+$R$^+$ \\
5   & L$^+$S$^+$L$^+$L$^-$      & 5         & L$^+$S$^+$L$^-$ \\
6   & L$^+$S$^+$R$^+$R$^-$      &  6         & L$^+$S$^+$R$^-$ \\
7   & R$^+$S$^+$L$^+$L$^-$       &  7         & R$^+$S$^+$L$^-$ \\
8   & R$^+$S$^+$R$^+$R$^-$      &  8         & R$^+$S$^+$R$^-$ \\
9   & L$^-$L$^+$S$^+$L$^+$      &  9         & L$^-$S$^+$L$^+$ \\
10  & L$^-$L$^+$S$^+$R$^+$      &  10        & L$^-$S$^+$R$^+$  \\
11  & R$^-$R$^+$S$^+$L$^+$      &  11        & R$^-$S$^+$L$^+$  \\
12  & R$^-$R$^+$S$^+$R$^+$       &  12        & R$^-$S$^+$R$^+$ \\
13  & L$^-$L$^+$S$^+$L$^+$L$^-$   &  13        & L$^-$S$^+$L$^-$  \\
14  & L$^-$L$^+$S$^+$R$^+$R$^-$    &  14        & L$^-$S$^+$R$^-$ \\
15  & R$^-$R$^+$S$^+$L$^+$L$^-$   &  15        & R$^-$S$^+$L$^-$ \\
16  & R$^-$R$^+$S$^+$R$^+$R$^-$  &  16        & R$^-$S$^+$R$^-$ \\
17  & L$^-$R$^-$L$^-$          &  17        & L$^+$R$^+$L$^+$ \\ 
18 & R$^-$L$^-$R$^-$ & 18 & R$^+$L$^+$R$^+$ \\
19 & L$^+$L$^-$L$^+$L$^+$ & 19 & L$^+$R$^+$L$^-$ \\
20 & L$^+$L$^-$L$^+$R$^+$ & 20 & R$^+$L$^+$R$^-$ \\
21 & R$^+$R$^-$R$^+$L$^+$ & 21 & L$^+$R$^-$L$^+$ \\
22 & R$^+$R$^-$R$^+$R$^+$  & 22 & R$^+$L$^-$R$^+$ \\
23 & L$^+$L$^+$L$^-$L$^+$ & 23 & L$^+$R$^-$L$^-$ \\
24 & L$^+$R$^+$R$^-$R$^+$ & 24 & R$^+$L$^-$R$^-$ \\
25 & R$^+$L$^+$L$^-$L$^+$  & 25 & L$^-$R$^+$L$^+$ \\
26 & R$^+$R$^+$R$^-$R$^+$ & 26 & R$^-$L$^+$R$^+$ \\
27 & L$^+$L$^-$L$^+$L$^+$L$^-$ & 27 & L$^-$R$^+$L$^-$ \\
28 & L$^+$L$^-$L$^+$R$^+$R$^-$  & 28 & R$^-$L$^+$R$^-$ \\
29 & R$^+$R$^-$R$^+$L$^+$L$^-$ & 29 & L$^-$R$^-$L$^+$ \\
30 & R$^+$R$^-$R$^+$R$^+$R$^-$ & 30 & R$^-$L$^-$R$^+$ \\
31 & L$^-$L$^+$L$^+$L$^-$L$^+$ & 31 & L$^-$R$^-$L$^-$ \\
32 & L$^-$L$^+$R$^+$R$^-$R$^+$ & 32 & R$^-$L$^-$R$^-$ \\
33 & R$^-$R$^+$L$^+$L$^-$L$^+$ & & \\
34 & R$^-$R$^+$R$^+$R$^-$R$^+$ & & \\
\bottomrule
\end{tabular}\vspace{-12pt}
\begin{tablenotes}
    \centering
    \item[1] {}
\end{tablenotes}
\end{table}

\begin{figure*}[t!]
    \centering
        \includegraphics[width=1\textwidth]{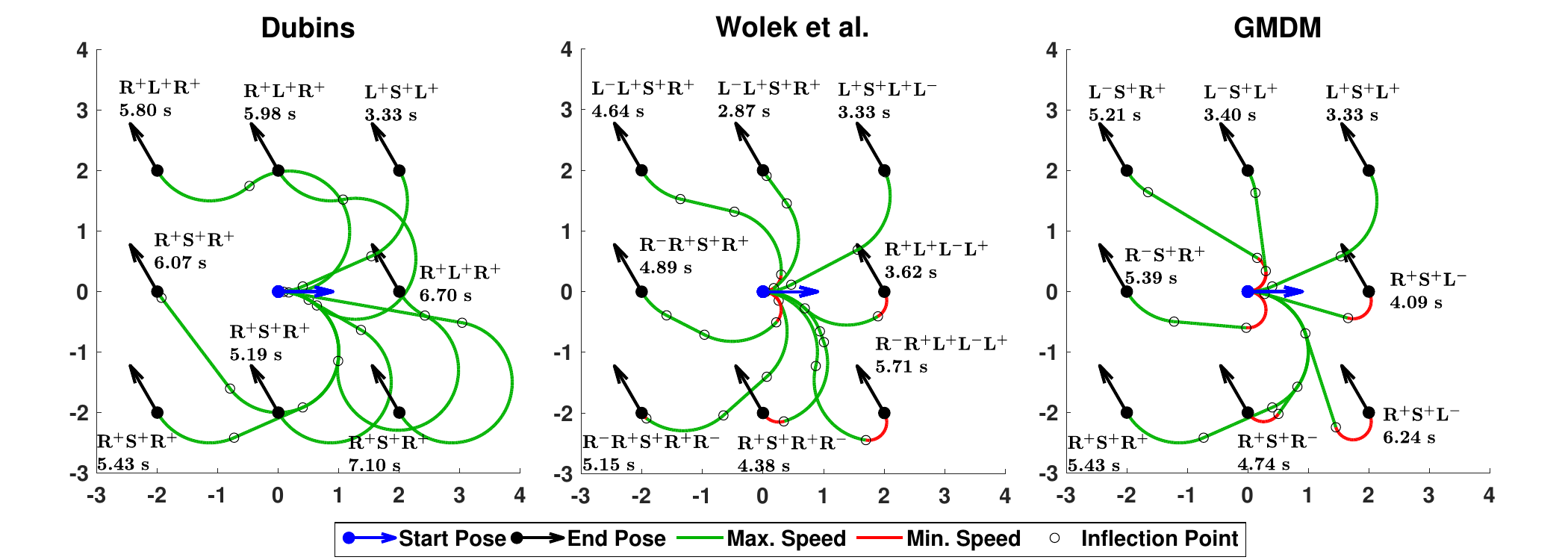}
         \caption{Visualization of paths produced by the three motion models (Dubins, Wolek's and GMDM$'$-$2$) from $\textbf{p}_0=(0,0,0)$ to different goal poses.}
         \label{fig:optimal_path_ex}\vspace{-6pt}
\end{figure*}

\vspace{6pt}
\subsubsection{The Wolek's model}
Wolek et al. \cite{WCW16} expanded the Dubins model for time-optimal planning of multi-speed vehicles. However, this model uses only the extremal (i.e., min and max) speeds which are shown to be sufficient for time-optimal planning in obstacle-free environments. As such, this model might produce sub-optimal results for time-risk optimal planning in obstacle-rich environments. Specifically, in this model, the L and R segments are of two types: bang ($^+$) and cornering ($^-$), which correspond to the max and min speeds, respectively. The S segments are always at the max speed. Clearly, the speeds other than min and max might be necessary on the turn segments to wrap tightly around the obstacles and on the straight line segments to reduce risk. Furthermore, there is no restriction on the number of segments for a path type. For example, the path L$^-$L$^+$S$^+$R$^+$R$^-$ means that the vehicle first turns left at min speed, then turns left at max speed, then continues on a straight line at max speed, then turns right at max speed, and finally turns right at min speed before reaching the goal. This kind of complex maneuver might be difficult to follow by the vehicle. Wolek et al. \cite{WCW16} provided a sufficient set of $84$ candidate path types that is guaranteed to yield the time-optimal path to any goal pose in an obstacle-free environment. For simplicity, Wolek et al. \cite{WCW16} further provided a smaller set of $34$ candidate path types that is most likely to yield the time-optimal path. Let $\Gamma_W$ denote this set of most-likely candidate paths \cite{WCW16}, as shown in Table~\ref{tab:wolek34}. Unlike Dubins, the solution of Wolek's model requires nonlinear optimization, thus making it impractical for onboard real-time implementation.

\vspace{6pt}
\subsubsection{GMDM}
This is an extension of the Dubins model that allows multi-speed configurations of the Dubins path types by selecting a different control input $\textbf{u}_i=(v_i,\omega_i)$ on each path segment $i=1,2,3$. For practical implementation, the set of GMDM path types is obtained by discretizing the control inputs for each path segment. Thus, we define the following:

{
\begin{defn} [GMDM-$k$]
    Let $k\in\mathbb{N}^+$ be the number of speeds to consider.  Let $\mathbb{V}_k=\{v_{\ell} \in [v_{min},v_{max}]: \ell=1,...k\}$ be the set of $k$ uniformly-spaced speeds, where
    \begin{equation}
        v_\ell = \begin{cases} 
        v_{max} &  k=1\\
        v_{min} + (\ell-1)\frac{v_{max}-v_{min}}{k-1} & k>1.
         \end{cases}
    \end{equation}
   Then, GMDM-k is defined as GMDM, where for each path segment i=1,2,3, a) the speed $v_i\in\mathbb{V}_k$, and b) the turning rate $|\omega_i|\in\{0,\omega_{max}\}$. Note: GMDM-1 denotes the Dubins model.
   
\end{defn}
}

\begin{rem}
 In general, there could be other ways to select speeds for the GMDM path segments, but these approaches are not discussed here.
\end{rem}

Let $\Gamma_{Gk}$ denote the set of GMDM-$k$ path types. Then, the total number of  GMDM-$k$ path types is $|\Gamma_{Gk}|$=$6k^3$. 

\begin{rem}
The set $\Gamma_D$ of Dubins path types is contained in the set $\Gamma_{Gk}$ of GMDM-$k$ path types, $\forall$ $k\in \mathbb{N}^+$, (i.e., $\Gamma_D\subseteq\Gamma_{Gk}$). The  equality holds for $k=1$ (i.e., $\Gamma_D=\Gamma_{G1}$). Thus,
the solution quality of GMDM paths in terms of time/time-risk costs is guaranteed to be better than or the same as the Dubins paths.

\end{rem}

{
\begin{defn} [GMDM$'$-$k$] 
   GMDM$'$-$k$ is  GMDM-k where the speed for the straight line segment is set to $v_{max}$.
\end{defn}
}

{ Let $\Gamma'_{Gk}\subset \Gamma_{Gk}$ be the set of GMDM$'$-$k$ path types. Then, the total number of  GMDM$'$-$k$ path types is $|\Gamma'_{Gk}|=2k^3 + 4k^2$.} 

{\begin{rem} For time-optimal planning we use GMDM$'$-$k$ since the straight line segments must be at max speed. However, for time-risk optimal planning we use GMDM-$k$ for minimizing time-risk costs. 
\end{rem}}

GMDM provides a closed form solution just like the Dubins model; thus, its computation is straightforward. 
Furthermore, the GMDM path types are simple with only three segments; thus, the time complexity of GMDM to obtain the solution of each individual path type is the same as that of Dubins. However, as compared to the Dubins model, GMDM has more path types depending on the number of possible speeds on each segment. Thus, the time complexity of GMDM typically falls between that of the Dubins and Wolek's models; however, GMDM computation is still real-time and significantly faster than the Wolek's model due to the closed-form solutions. 

While GMDM does not guarantee to produce time-optimal solutions in obstacle-free environments, the results later show that the solution quality of GMDM is significantly better than the Dubins model and approaches that of the Wolek's model (i.e., time-optimal solutions). Furthermore, the GMDM solutions are obtained in real-time with orders of magnitude faster computation than the Wolek's model. On the other hand, neither GMDM nor the baseline models guarantee to provide time-risk optimal solutions in obstacle-rich environments; however, the results later show that GMDM in fact produces significantly better solution quality (i.e., time-risk cost) than both baseline models. { This is because the Dubins model with single speed and the Wolek's model with two speeds cannot effectively minimize the time-risk cost, while GMDM with multiple speeds is capable of doing that.}

The Dubins model, the Wolek's model and GMDM use the path types from $\Gamma_D$, $\Gamma_W$, and $\Gamma_{Gk}$, respectively, to compute the different possible paths between any pair of poses. Then, each model selects the path with the least cost as the final solution.

\vspace{-6pt}
\subsection{Time-Optimal Planning}
\label{sec:Obstacle-freeEnv} In this section, we present the comparative evaluation results for time-optimal planning in obstacle-free and obstacle-rich environments without considering risk.  
According to the Wolek's model \cite{WCW16}, the time-optimal paths in obstacle-free environments can be generated using the extremal speeds (i.e., $v_{min}$ and $v_{max}$). Thus, for time-optimal planning we consider GMDM$'$-$2$ and its set $\Gamma'_{G2}\subset \Gamma_{G2}$ of path types that are built from bang (i.e., max speed) and cornering (i.e., min speed) arcs and max speed straight line segments. For notational simplicity we refer to GMDM$'$-$2$ as GMDM unless specified. Table~\ref{tab:wolek34} shows the set $\Gamma'_{G2}$ that consists of $32$ GMDM$'$-$2$ path types.

\vspace{6pt}
\subsubsection{Time-optimal planning in obstacle-free environments} 
First, we consider time-optimal planning in obstacle-free environments. We visualize the paths generated by different motion models and perform Monte Carlo simulations to compare their solution quality (i.e., travel time) and computation time.

\begin{figure}[t!]
    \centering
\subfloat[{Travel times for each motion model.}]{
 \includegraphics[width=0.36\textwidth]{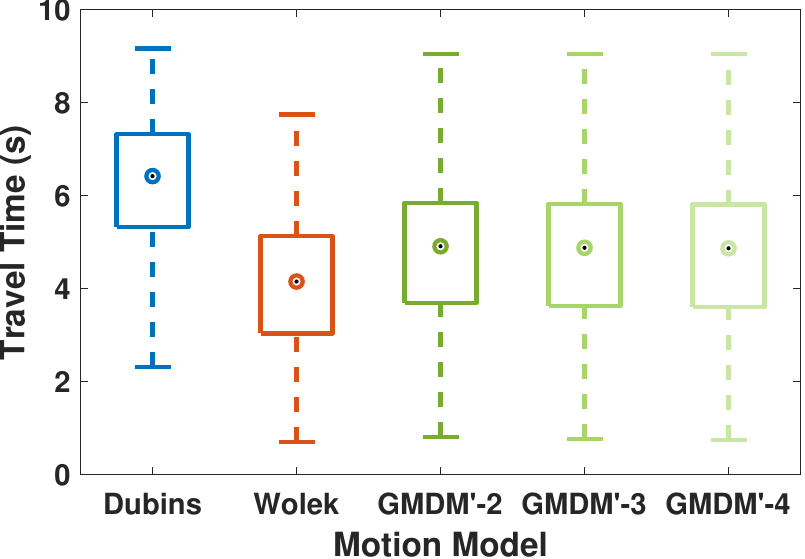}\label{fig:numspeedstudy_a}
 }\\\vspace{5pt}\hspace{-9pt}
 \subfloat[{Average computation times for each motion model.}]{
 \includegraphics[width=0.37\textwidth]{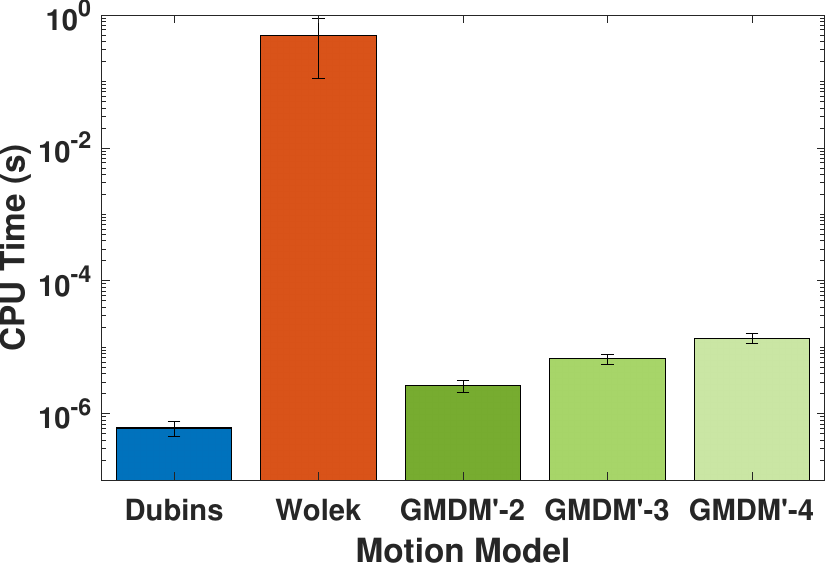}\label{fig:numspeedstudy_b}
 }
\caption{{ Comparative evaluation results of the three motion models for time-optimal planning in an obstacle-free environment. The results show the effects of the number of speeds of GMDM on the solution quality and computation time.}
         }   \vspace{-6pt}
\label{fig:numspeedstudy}
\end{figure}

\vspace{6pt}
$\bullet$ \textit{Visualization of paths}:
To visualize and compare the paths produced by the three motion models (i.e., the Dubins model, the Wolek's model and GMDM$'$-$2$), we consider different goal poses around the origin with the start pose $\textbf{p}_0$=$(0,0,0)$. Fig.~\ref{fig:optimal_path_ex} shows the paths produced by the three motion models to the goal poses and their travel times. Note that for simplicity GMDM$'$-$2$ is referred as GMDM. 
As seen, the Dubins paths to the goals are indirect, i.e., long and curvy,  due to the single-speed restriction. However, the Dubins model takes only $\sim 6\times 10^{-7}$~s of computation time to find the solution. 
The Wolek's paths, on the other hand, are time-optimal with sharp turns and shortest travel times.
However, the Wolek's model has a higher average computation time of $\sim 3.57\times 10^{-1}$~s. Finally, the GMDM paths are much more direct than the Dubins paths and approach the solution quality of the Wolek's paths. GMDM takes only $\sim 3\times 10^{-6}$~s to find the solution; thus, it is well-suited for real time applications. 
Taking an example, say the goal pose to the right of the start pose, the Dubins, Wolek's and GMDM paths have the travel time costs of $\sim6.70$~s, $\sim3.62$~s and $\sim4.09$~s, respectively.

\vspace{6pt}
$\bullet$ \textit{Comparative evaluation results}:
For further evaluation, we conducted Monte Carlo simulations for time-optimal planning from the start pose at the origin to 5000 randomly generated end poses. The $x-y$ positions of these end poses were generated from a uniform distribution over a disk of $3$ m radius centered at the origin, whereas the heading angle $\theta$ is generated from a uniform distribution over the interval $[0,2\pi)$.

\begin{figure}
    \centering
    \includegraphics[width=0.44\textwidth]{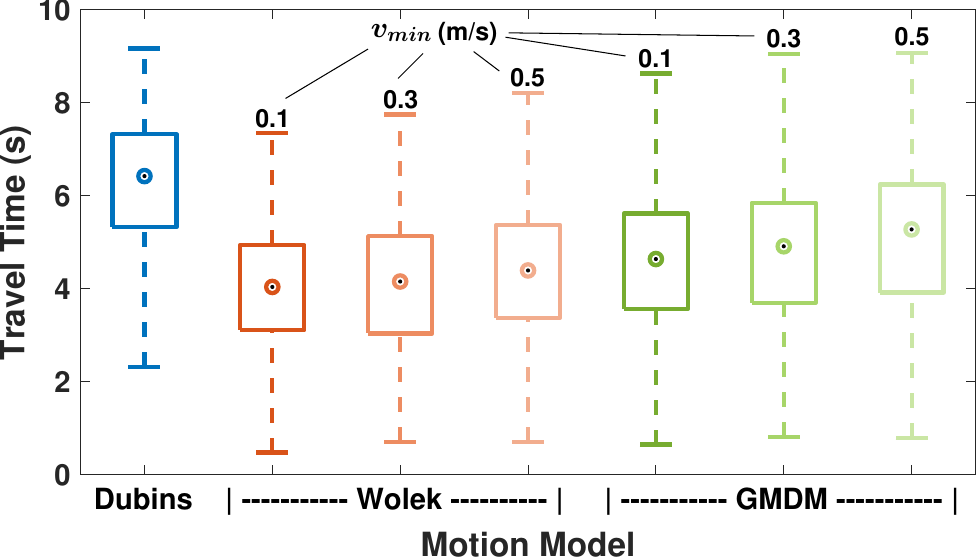}
    \caption{{  Comparative evaluation results showing the effects of different $v_{min}$ on the solution quality of the Wolek's model and GMDM. Note: $v_{max}=1.0$ m/s for all three models.}} \vspace{-12pt}
    \label{fig:minspeedstudy}
\end{figure}

\begin{figure*}[t!]
    \centering
    \subfloat[Visualization of paths produced by the three motion models (Dubins, Wolek's and GMDM$'$-$2$) for a TSP with $3$ points of interest.]{
        \includegraphics[width=0.9\textwidth]{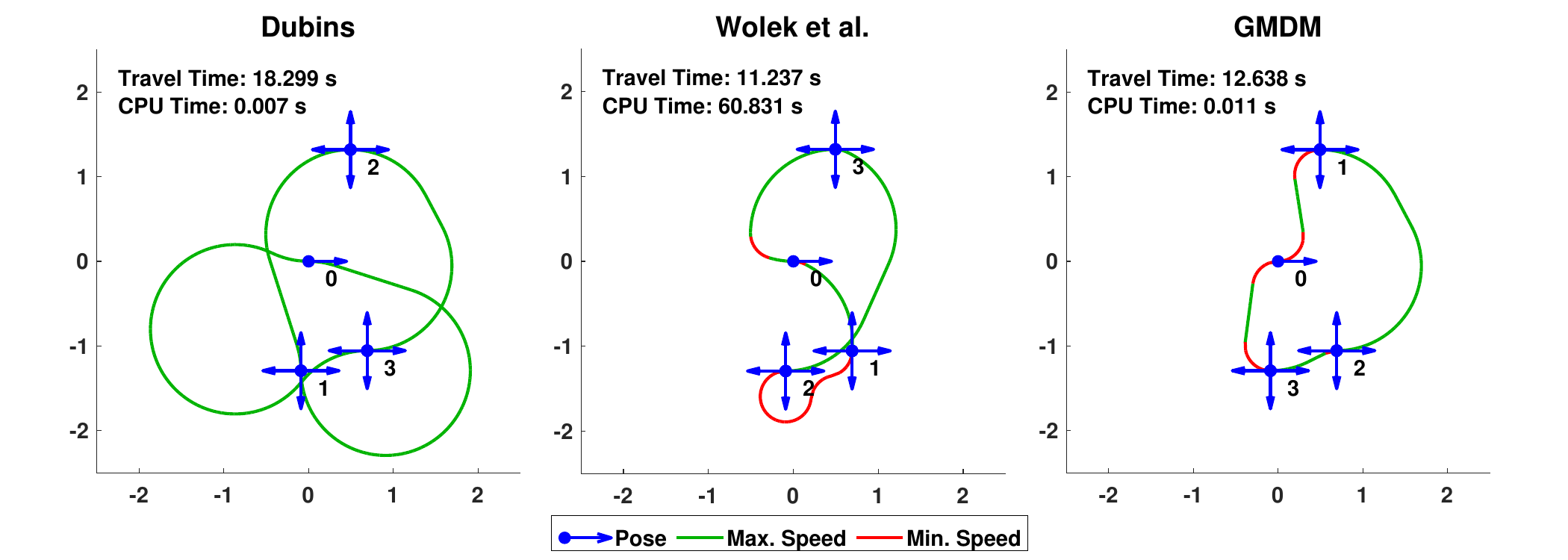}\label{fig:DubinsTSP}
        }\\
 \subfloat[Comparative evaluation results of the three motion models for a TSP with $5$ points of interest.]
 {
 \includegraphics[width=0.65\columnwidth]{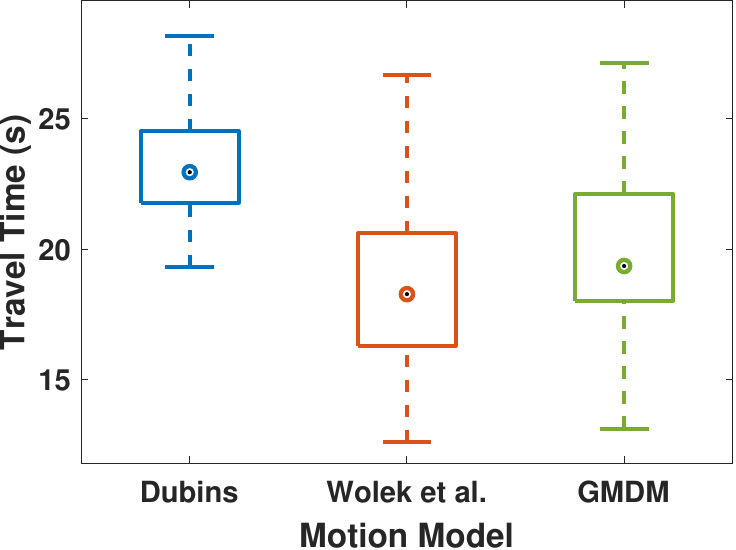}
 \label{fig:tspmontecarlo_a} \hspace{12pt} \includegraphics[width=0.66\columnwidth]{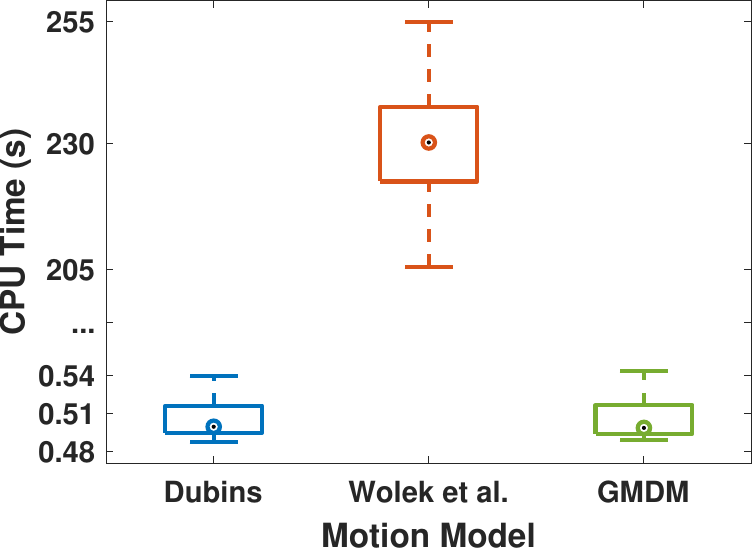}
 \label{fig:tspmontecarlo_b}
 }
        \caption{Comparison of TSP solutions produced by the three motion models (Dubins, Wolek's and GMDM$'$-$2$). 
         }   \vspace{-6pt}
\end{figure*}

The first study examines the effects of the number of speeds considered in GMDM on the solution quality (i.e., travel time) and computation time. In particular, the GMDM paths are selected from the sets $\Gamma'_{Gk}$, where $k=1,\dots,4$. Thus, the models compared are GMDM$'$-$1$ (the Dubins model), the Wolek's model, GMDM$'$-$2$,  GMDM$'$-$3$, and  GMDM$'$-$4$. 
Fig.~\ref{fig:numspeedstudy_a} shows the statistical results of the travel times via box plots. The dot marks the median, the box shows the middle \text{$50^{th}$} percentile and the horizontal lines show the min and max values.  Fig.~\ref{fig:numspeedstudy_b} shows the corresponding average computation times. As seen, GMDM$'$-$1$ (the Dubins model) provides the worst solution quality as expected, with a median travel time of $\sim 6.42$ s; however, it has the fastest computation time of $\sim 6\times 10^{-7}$ s. Next, the Wolek's model provides the time-optimal solution quality, with a median travel time $\sim 4.15$ s; however, it has the slowest computation time of $\sim 0.5$ s. On the other hand, GMDM$'$-$2$ provides much better solution quality than Dubins and approaches that of Wolek, yielding a median travel time of $\sim 4.91$ s, while needing $\sim 2.64\times 10^{-6}$~s to get a solution. GMDM$'$-$3$ and GMDM$'$-$4$ provided median travel time results of $\sim 4.88$ s and $\sim 4.87$ s, respectively, which are only slightly better than GMDM$'$-$2$; however, GMDM$'$-$3$ and GMDM$'$-$4$ had larger computation times of $\sim 6.670\times 10^{-6}$ s and $1.36\times 10^{-5}$ s, respectively. Thus, at least for time-optimal planning, GMDM$'$-2 clearly provides the best trade-off for solution quality and computation time. Therefore, for the remaining results of time-optimal planning we use GMDM$'$-2 ($\Gamma'_{G2})$ and refer it as GMDM for simplicity.

{ The second study examines the effect of different values of $v_{min}$ on the solution quality of the Wolek's model and GMDM$'$-$2$, and the results are shown in Fig.~\ref{fig:minspeedstudy}. As expected, the Wolek's model provides the best performance for $v_{min}=0.1$ m/s, yielding a median travel time of $\sim 4.04$ s. As $v_{min}$ increases to $0.3$ m/s and $0.5$ m/s, the median travel times slightly increase to $\sim 4.15$ s and $\sim 4.39$ s, respectively. GMDM shows a similar trend, yielding the best median travel time of $\sim 4.64$ s for $v_{min}=0.1$ m/s. Again, as $v_{min}$ increases to $0.3$ m/s and $0.5$ m/s, the median travel times for GMDM slightly increase to $\sim 4.91$ s and $\sim 5.28$ s. Both the Wolek's model and GMDM outperform the Dubins model; however, their solution quality approach that of the Dubins solution as $v_{min}$ approaches $v_{max}$. 
Based on these results, we select $v_{min}=0.3$ m/s for the remainder of this paper as it reasonably demonstrates the benefits of GMDM and Wolek.

We also performed a final study examining the effects of having a different number of uniformly-spaced turning rates considered for each segment of GMDM. The results showed no solution quality benefit when considering turning rates besides $|\omega|\in\{0,\omega_{max}\}$; however, considering more than the extreme turning rates yielded higher computation times, as expected. For brevity's sake, these results figures are omitted.
}

Overall, considering the travel and computation time costs, the GMDM paths provide much better solution quality as compared to the Dubins paths while requiring similar computation times. On the other hand, GMDM has a significant computational advantage over the Wolek's solutions while achieving similar path quality. Furthermore, GMDM is easier to implement due to the closed-form solutions. Also, its paths with only three segments are simpler and more appealing  as compared to the Wolek's paths with more than three segments; thus, they are easier to follow by onboard controllers.

\begin{figure*}[!t]
\centering
    \subfloat[Visualization of paths produced by RRT$^*$ using the three motion models (Dubins, Wolek's and GMDM$'$-$2$) in an obstacle-rich scenario.]{
        \includegraphics[width=0.84\textwidth]{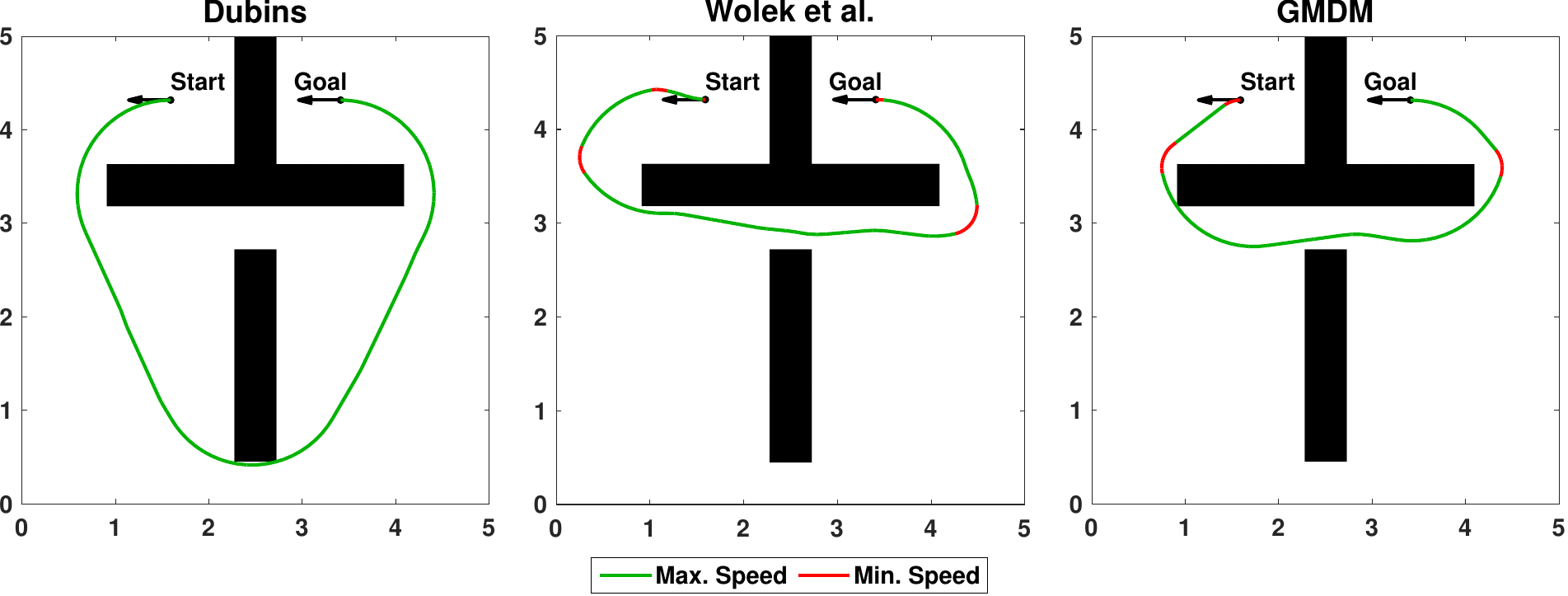}
        \label{fig:rrtstarpaths}
    }\\
    \centering
    \vspace{-6pt}
    \subfloat[Median performance vs. time.]{
        \includegraphics[width=0.275\textwidth]{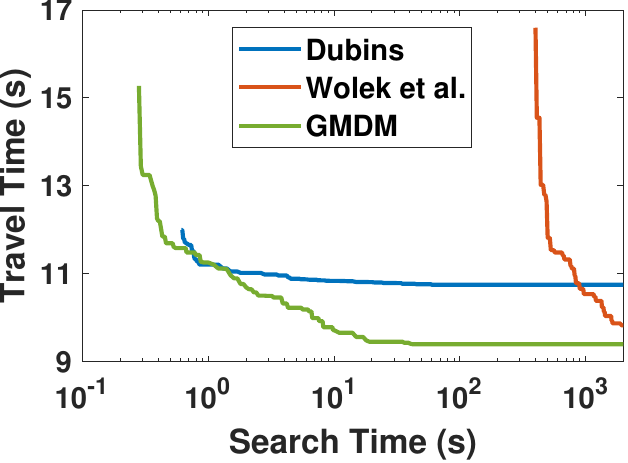}
        \label{fig:rrtstarconvrate}
    }
    \subfloat[Travel time cost.]{
        \includegraphics[width=0.275\textwidth]{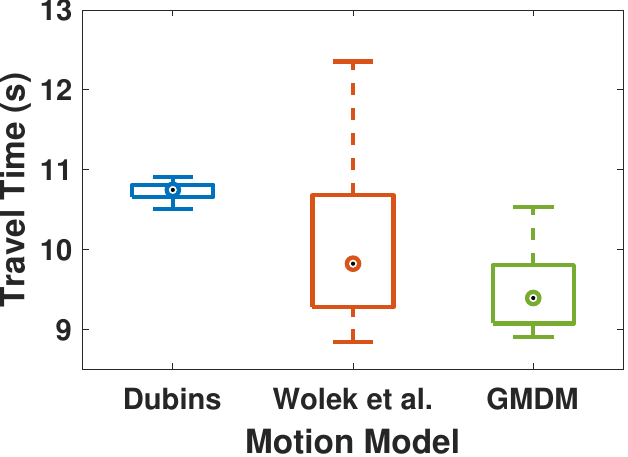}
        \label{fig:rrtstarboxplottimes}
    }
    \subfloat[Number of nodes in the search tree.]{
        \includegraphics[width=0.275\textwidth]{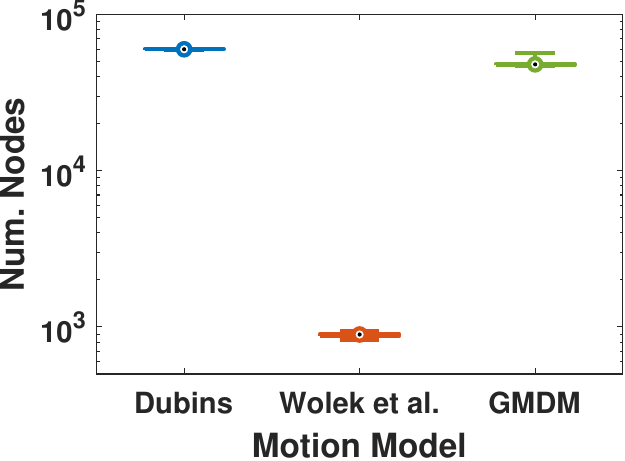}
        \label{fig:rrtstarboxplotnodes}
    }
\caption{Comparative evaluation results for time-optimal planning in obstacle-rich environments using RRT$^*$ with the three motion models (Dubins, Wolek's and GMDM$'$-$2$).}
\label{fig:TOP_rrtstar}  \vspace{-6pt}  
\end{figure*}

\vspace{6pt}

$\bullet$ \textit{Application to Traveling Salesman Problem}:
The above analysis provided insights into the types of paths produced by the three motion models along with the comparative evaluation of their travel and computation times for single destination problems. However, many path planning problems often need to consider the combinations of different intermediate waypoints before reaching the final destination. One such problem is the Dubins traveling salesman problem (TSP) \cite{Savla2008_DubinsTSP,Bull02011_DubinsTSP,Ny2012_DubinsTSPDiscrete,Cohen2017_DubinsTSPProgram} which is stated as follows: given a collection of $n$ points that must be visited, find (1) the sequence of points and (2) the heading at each point in the sequence that yields the time-optimal path for a curvature-constrained vehicle. It is clear that the overall path quality and computation time of this combinatorial optimization problem depends on the quality and computation time of the underlying motion model.

We consider the scenario shown in Fig.~\ref{fig:DubinsTSP}. The starting point of the vehicle is at the origin with $\textbf{p}_0=(0,0,0)$, which is marked with the number "0". There are three points of interest that must be visited after which the vehicle must return to the start pose. At each point, four headings \{$0$, $\pi/2$, $\pi$, $3\pi/2$\}rad are considered. The objective is to find the minimum time path that visits each of these points at a certain heading and finally returns back to the origin. To solve TSP using any underlying motion model, we have to first compute the paths and their time costs for every pose pair between different points of interest. A pose at a point of interest can connect to $8$ other poses. Thus, for a total of $12$ poses, there are $12 \times 8 =96$ pose pairs. Additionally, there are $12 + 12=24$ connections between the start pose and the other poses and vice versa. Thus, there are a total of $96+24=120$ pose pairs. 
For each motion model, once the paths and their time costs are computed for the above pose pairs, the optimal solution of TSP is found by searching for the sequence that has the least total travel time cost.

Fig.~\ref{fig:DubinsTSP} visualizes the paths produced by the three motion models (i.e., the Dubins model, the Wolek's model and GMDM$'$-$2$) for a TSP with $3$ points of interest. The numbers $1$, $2$ and $3$ indicate the order in which these points are visited. The Dubins path takes long and clumsy routes to connect each of the points of interest, resulting in a total travel time of $\sim 18.30$~s. While the Dubins path is the longest, it is computed in the fastest time of $\sim 0.007$~s. The Wolek's solution, on the other hand, provides a superior path quality, with direct paths connecting the points of interest with a total travel time of $\sim 11.24$~s. However, it takes $\sim 60.83$~s of computation time, which might not be acceptable for dynamic path planning problems. Finally, the quality of GMDM TSP solution is very close to the Wolek's solution with a total travel time of $\sim 12.64$~s, but it is obtained significantly faster in $\sim 0.011$~s.

For further validation, we conducted a Monte Carlo study of a TSP with $5$ points of interest. { For this study, 30 scenarios were considered.} For each scenario, five points of interest were randomly distributed around the start pose within a $5$~m radial distance. Each point has four possible headings as before. The objective is to find the minimum time path that travels through each of these points and returns back to the origin. Fig.~\ref{fig:tspmontecarlo_b} shows the statistical comparison results by box plots of the travel and computation times for the above Monte Carlo simulations. 
Fig.~\ref{fig:tspmontecarlo_a} shows that the GMDM's solution quality (i.e., travel time) is superior to the Dubins solutions and approaches that of the Wolek's time-optimal solutions. At the same time, Fig.~\ref{fig:tspmontecarlo_b} shows that GMDM is significantly faster, similar to Dubins, and took only $\sim0.5$~s to get the TSP solution as compared to the Wolek's model which required $\sim230$~s to get the solution.

Overall, the above results demonstrate that in obstacle-free environments:  (1) GMDM provides superior path quality as compared to the Dubins model while approaching the quality of the time-optimal Wolek's paths, and (2) GMDM has very low computational requirements similar to Dubins while being significantly faster than the Wolek's model; thus it is suitable for real-time path planning and replanning applications~\cite{SMART_Shen2023}.

\vspace{6pt}
\subsubsection{Time-optimal planning in obstacle-rich environments}\label{sec:MultispeedDubinsT*}
Next, we consider time-optimal planning in obstacle-rich environments without considering risk. We visualize the paths generated by different motion models and perform Monte Carlo simulations to compare their solution quality (i.e., travel time) and computation time. 
For obstacle-rich scenarios a high-level planner is needed to compute the time-optimal path. This planner uses the Dubins, Wolek's, and GMDM as the underlying motion models to connect any pair of sample poses. { In particular, we use RRT$^*$ \cite{Karaman2011_RRTPRM*} which is an asymptotically-optimal sampling-based motion planner that can produce time-optimal paths as long as the underlying motion model can produce time-optimal paths between any two neighboring sampled poses. Thus, RRT$^*$ with the Dubins model cannot produce time-optimal paths for multi-speed vehicles. On the other hand, RRT$^*$ with the Wolek's model can produce time-optimal paths asymptotically. Finally, since GMDM is only near-optimal, RRT$^*$ with GMDM is not guaranteed to produce time-optimal paths in the limit. However, since GMDM is orders of magnitude faster than the Wolek's model, RRT$^*$ with GMDM converges significantly faster than with the Wolek's model. This is because faster execution of GMDM allows RRT$^*$ to place more nodes in the total allotted time and find a better solution as compared to the Wolek's model. } 

\vspace{6pt}
$\bullet$ \textit{Visualization of paths}:
Fig.~\ref{fig:rrtstarpaths} shows an obstacle-rich scenario with the start and goal poses. As seen in Fig.~\ref{fig:rrtstarpaths}, the Dubins model produces a long path around the bottom obstacle due to its inability to make a sharp turn. The Wolek's model produces a better path that goes around the top obstacle to reach the goal. This path has a shorter travel time since it turns rapidly at slow speed and goes between the two obstacles. Finally, GMDM produces the best path that wraps around the top obstacle to quickly reach the goal in the fastest travel time. 

\vspace{6pt}
$\bullet$ \textit{Comparative evaluation results}:
Since RRT$^*$ places random samples in the environment, { $100$ Monte Carlo simulations were performed for comparative evaluation of the three motion models}.
Fig. \ref{fig:rrtstarconvrate} shows the convergence plot for each motion model. As seen, the Dubins plot converged the fastest, although the solution quality is low. On the other hand, the Wolek's plot did not converge in the allotted time, thus yielding a mediocre solution quality. Finally, the GMDM plot converged orders of magnitude faster than the Wolek's plot, while yielding the best solution quality as compared to both the baseline models. Figs. \ref{fig:rrtstarboxplottimes} and \ref{fig:rrtstarboxplotnodes} show the results of Monte Carlo simulations on this scenario. Fig. \ref{fig:rrtstarboxplottimes} shows the travel time box plots of the three models. Clearly, GMDM produces the overall best travel times while the Dubins paths took the longest travel times. { Fig. \ref{fig:rrtstarboxplotnodes} shows the number of nodes created using RRT$^*$. Clearly, GMDM  enabled the creation of a large number of nodes by RRT$^*$ due to faster computation and thus yielding better solution quality in the allotted computation time. On the other hand, the Wolek's model enabled the creation of much smaller number of nodes by RRT$^*$ due to its high computation time, thus yielding slow convergence. Overall, the results show that RRT$^*$ with GMDM provides the best solution quality in the least amount of time, thus making it suitable for practical applications. }

\vspace{-6pt}
\subsection{Time-Risk Optimal Planning}
In this section, we present the comparative evaluation results for time-risk optimal planning in obstacle-rich environments. 

\vspace{6pt}
\subsubsection{Motivation for time-risk optimal planning}
As discussed earlier, the joint time-risk optimal planning becomes critical in obstacle-rich environments to produce short but also safe paths. However, this requires flexibility in changing speed along the path to reduce risks near obstacles or through narrow passages yet providing reasonable travel times. 

In this paper, we use the T$^\star$ planner \cite{Song2019_tstar} to find approximate time-risk optimal paths in obstacle-rich environments. T$^\star$ uses the Dubins, Wolek's, and GMDM as the underlying motion models to generate the candidate paths between two neighboring poses. The poses are placed in a uniformly discretized $\text{SE}(2)$ configuration space. Then, the time-risk cost is computed for all candidate paths between any two neighboring poses as described in Section~\ref{sec:probform}. The risk-free collision threshold is $t^*=3$ s and the risk-weight is $\lambda=2$. Finally, the A$^*$ search is performed to find the overall time-risk optimal path. We refer the reader to \cite{Song2019_tstar} for more details on T$^\star$. Note: T$^\star$ is a deterministic planner, unlike RRT$^*$. Thus, only one trial is needed to obtain the results per motion model.

\begin{figure}[t!]
       \centering
        \includegraphics[width=0.95\columnwidth]{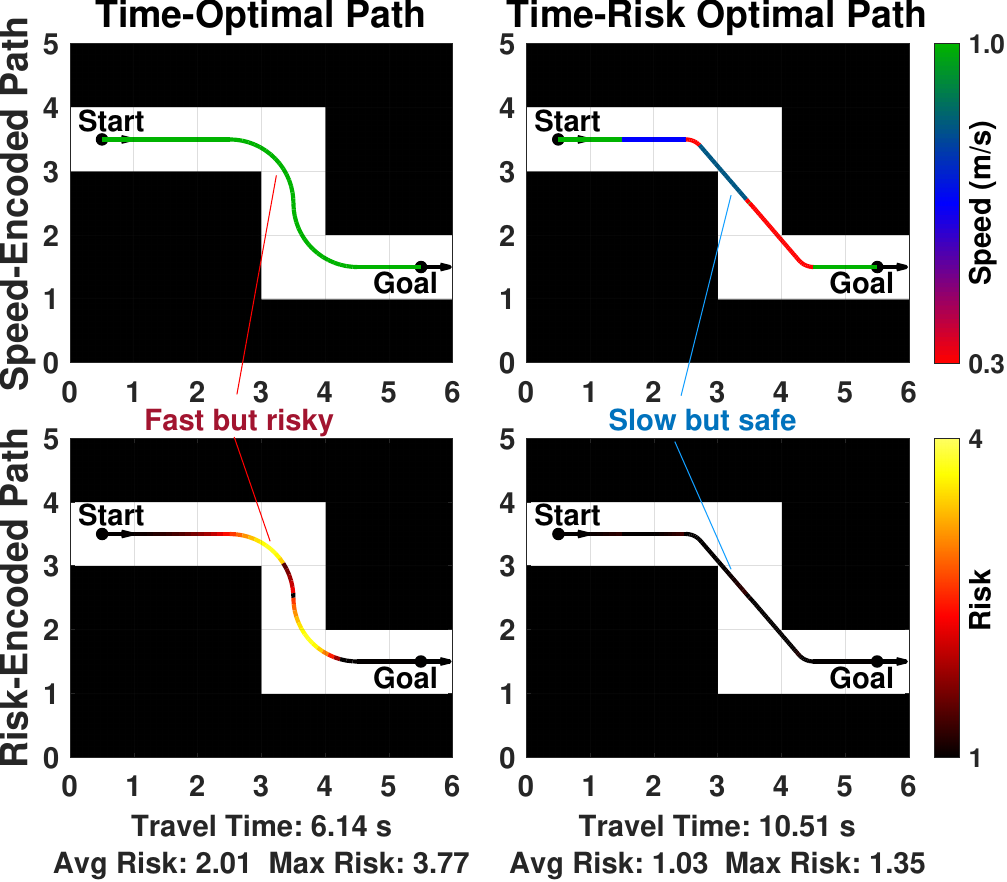}
        
    \caption{Visualization of the GMDM-$3$ paths encoded with the speed and risk information. The T$^\star$ \cite{Song2019_tstar} planner is used for time-optimal and time-risk optimal planning.
    }
    \label{fig:gmdmtstar_results1}\vspace{-12pt}
\end{figure}

\begin{figure*}[t!]
    \centering
    \subfloat[GMDM-1 (Dubins).]{
        \includegraphics[width=0.28\textwidth]{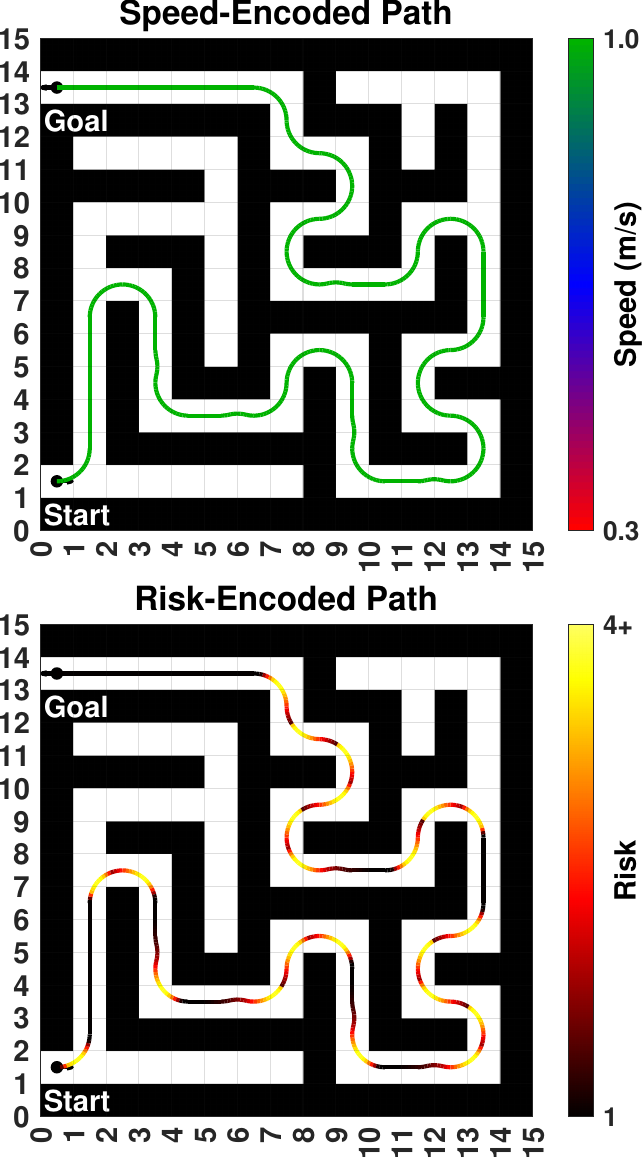}
        \label{fig:tstarmaze_dubins}
        }
    \subfloat[Wolek et al.]{
        \includegraphics[width=0.28\textwidth]{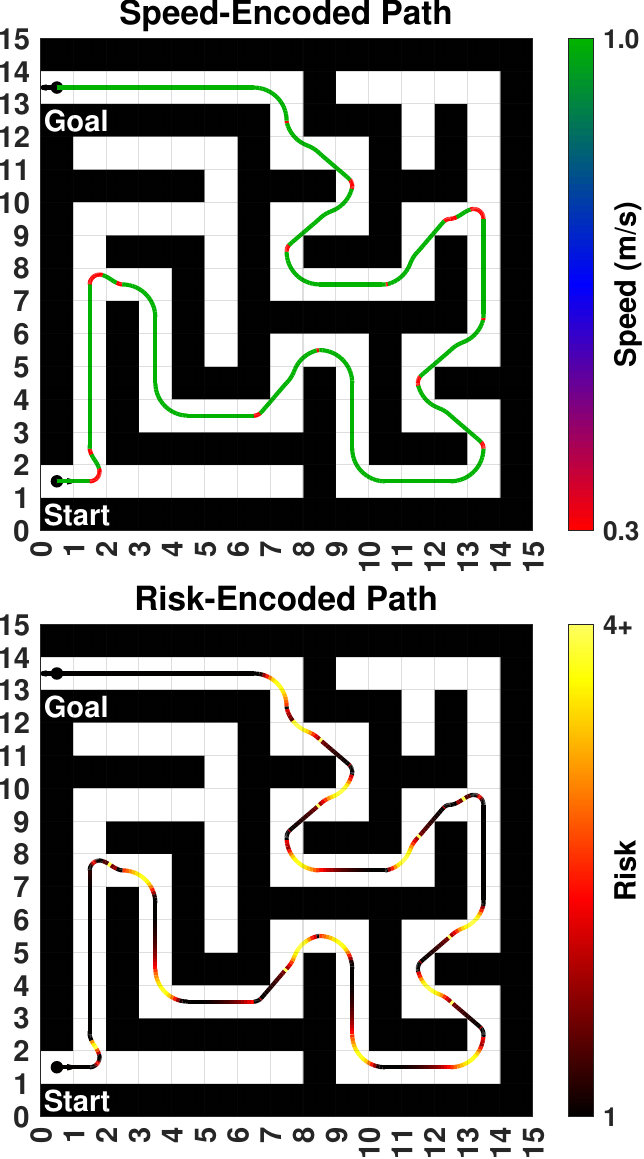}
        \label{fig:tstarmaze_wolek}
        }
    \subfloat[GMDM-2.]{
        \includegraphics[width=0.28\textwidth]{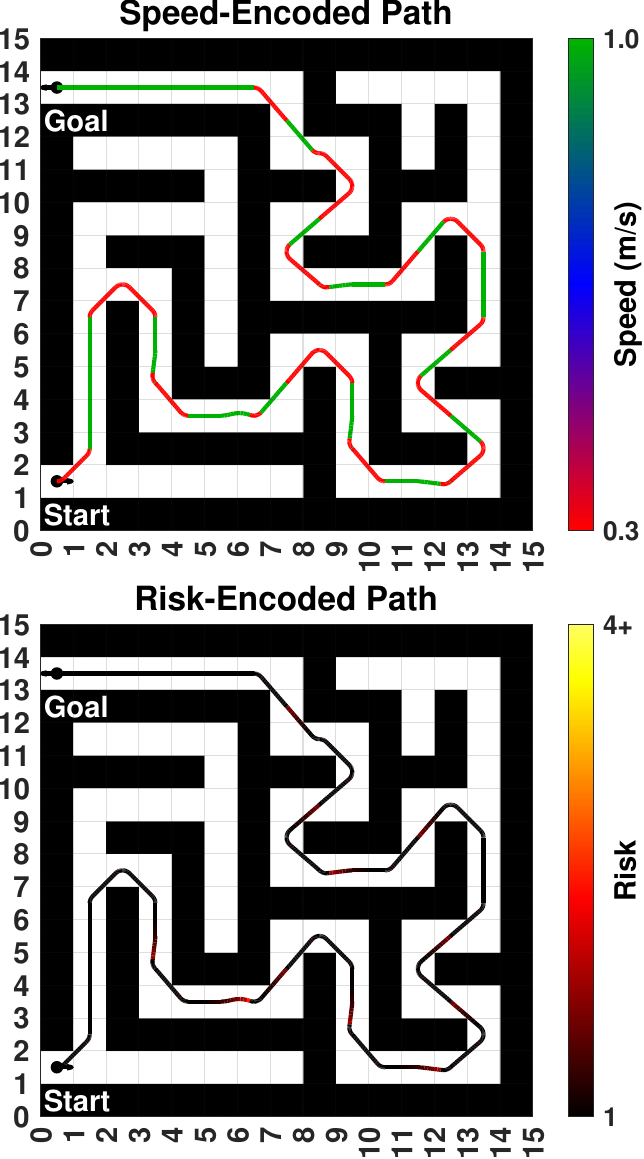}
        \label{fig:tstarmaze_gmdm2}
        }\\
    \subfloat[GMDM-3.]{
        \includegraphics[width=0.28\textwidth]{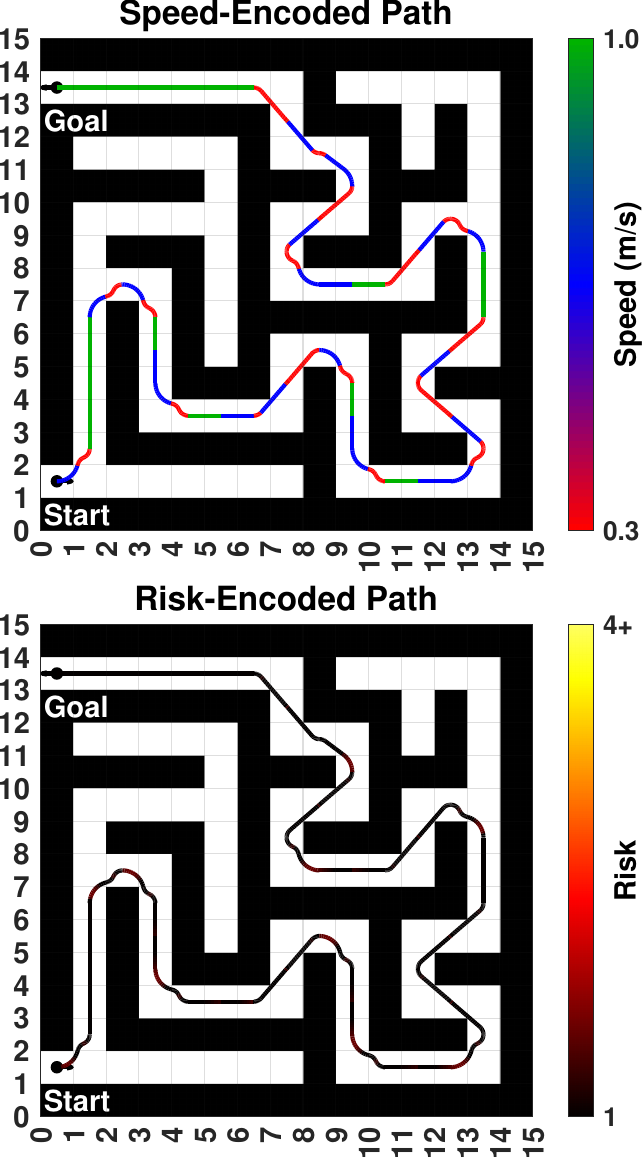}
        \label{fig:tstarmaze_gmdm3}
        }
    \subfloat[GMDM-4.]{
        \includegraphics[width=0.28\textwidth]{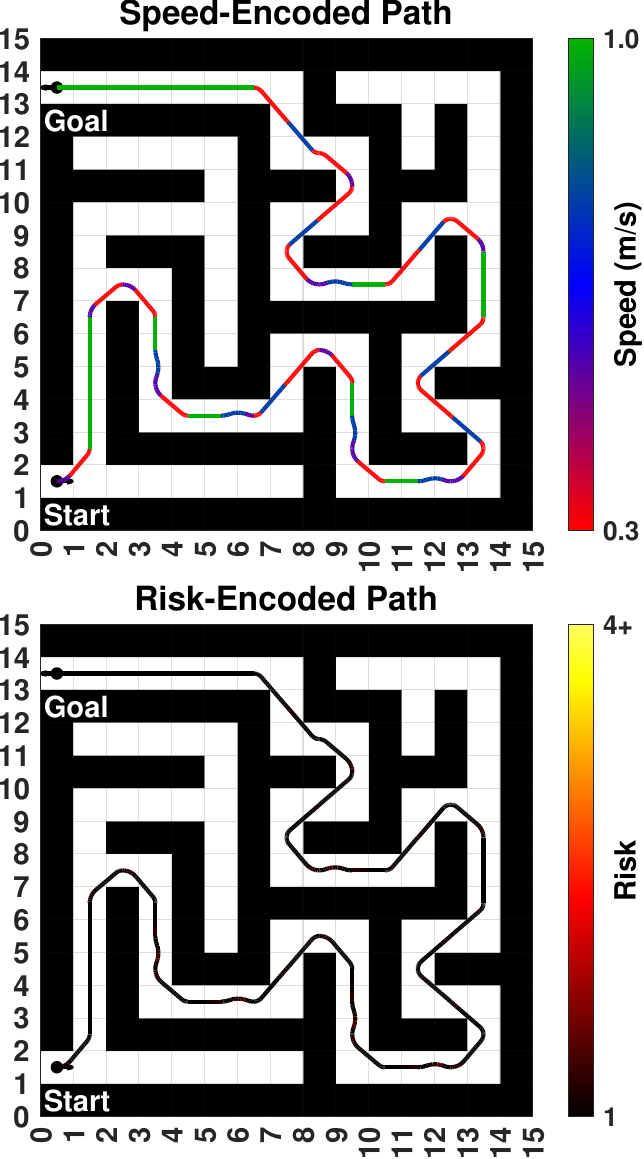}
        \label{fig:tstarmaze_gmdm4}
        }
    \subfloat[Results summary.]{
        \includegraphics[width=0.28\textwidth]{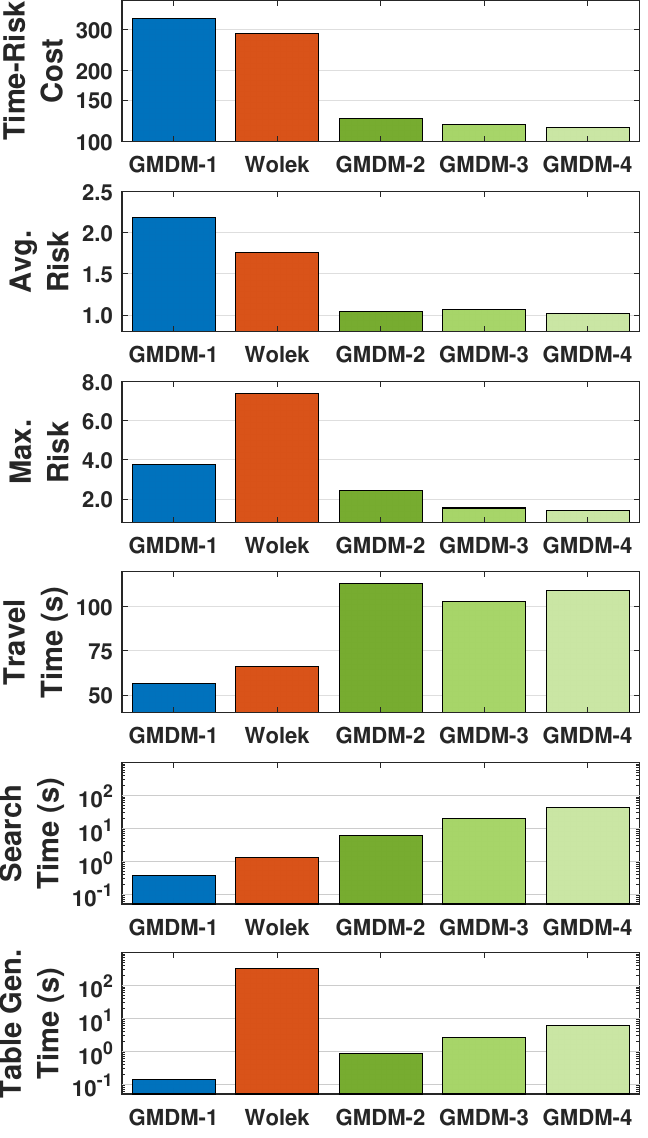}
        \label{fig:tstarmaze_results}
        }
        \caption{Comparison of the time-risk optimal paths generated by T$^\star$ \cite{Song2019_tstar} using different motion models (GMDM-1 (the Dubins model), the Wolek's model, GMDM-2, GMDM-3 and GMDM-4).}
        \label{fig:tstarmaze} \vspace{-12pt}
\end{figure*}

Fig.~\ref{fig:gmdmtstar_results1} visualizes the time-optimal and time-risk optimal paths produced by T$^\star$ \cite{Song2019_tstar} with GMDM-$3$ as the motion model. While the top row of Fig.~\ref{fig:gmdmtstar_results1} shows the speed-encoded time-optimal and time-risk optimal paths, the bottom row shows the corresponding risk-encoded paths. As seen in the top row, the time-optimal path runs only at the max-speed and yields a shorter travel time. On the other hand, the time-risk optimal path runs at three different speeds and yields a higher travel time. The time-risk optimal path chooses the min speed to make sharp turns and to reduce the collision risks. Thus, it runs at max speed when it is safe (i.e., when the approach time to an obstacle is high), min speeds during turning and in high-risk regions, and moderate speeds in other regions. As seen in the bottom row, the time-optimal path has high risk in most regions, while the time-risk optimal path has minimum risk and is safe. Fig.~\ref{fig:gmdmtstar_results1} shows that the average and max risks of the time-risk optimal path are significantly reduced.

\vspace{6pt}
\subsubsection{Comparative evaluation results}
As mentioned before, GMDM produces significantly better paths in terms of time-risk cost due to its added flexibility on choosing multiple speeds. While the Dubins paths have limited maneuverability because of the single speed constraint, the Wolek's paths also have limited capability in reducing the time-risk cost due to the use of only two speeds (i.e., min and max) and because the straight line segments are always at max speed, which increases the collision risk. 
In contrast, the GMDM paths enable appropriate speed selection to keep a balance between time and risk costs. In addition, the turning segments in GMDM can have multiple turning radii depending on the speeds, thus providing better maneuverability around obstacles.

Fig.~\ref{fig:tstarmaze} compares the time-risk optimal paths produced by GMDM-1 (the Dubins model), the Wolek's model, GMDM-2, GMDM-3 and GMDM-4. These models are implemented in the T$^\star$ planner \cite{Song2019_tstar} for the construction of the time-risk optimal paths in a complex maze scenario. For the T$^\star$ framework, the map is divided into cells of size $1\times 1$~m$^2$ with 8-connectivity. Each cell has 8 possible orientations that are evenly spaced at intervals of $\pi/4$. { The computations in T$^\star$ are done using an a priori look-up table generation of all possible path types between any pair of neighboring poses on the grid.} This speeds up the search time for path generation later.

The Dubins model used $v_{max}$, the Wolek's model used $v_{min}$ and $v_{max}$, and GMDM-2, GMDM-3 and GMDM-4 used 2, 3 and 4 speeds, respectively, which are uniformly-spaced in $[v_{min}, v_{max}]$. The results in Figs. \ref{fig:tstarmaze_dubins} and \ref{fig:tstarmaze_results} show that the GMDM-1 (the Dubins model) path travels fastest; however, it has the highest time-risk cost due to high average and max risks on the path. Figs. \ref{fig:tstarmaze_wolek} and \ref{fig:tstarmaze_results} show that the Wolek's path slows down at corners and reduces risk only slightly but is still very high due to the fast straight line segments, thus leading to an overall high time-risk cost. Figs. \ref{fig:tstarmaze_gmdm2} and \ref{fig:tstarmaze_results} show that the GMDM-2 path adapts two speeds throughout the entire path and hence able to reduce the risk significantly, thus leading to a much smaller time-risk cost. As seen, GMDM-2 compromises the travel time for better safety. The computation time of GMDM-2 is higher than that of Wolek because in T$^\star$ a look-up table is generated a priori containing all possible paths between any pair of poses in the local neighborhood on the grid. However, the computation time of GMDM-2 including the { a priori lookup table generation ("Table Gen." in Fig.~\ref{fig:tstarmaze_results})} and the search times, is smaller than the Wolek's model. Finally, Figs. \ref{fig:tstarmaze_gmdm3}-\ref{fig:tstarmaze_results} show that GMDM-3 and GMDM-4 adapt multiple speeds throughout the path to further reduce the risk and overall time-risk cost at the expense of slightly higher computation time. 

{ Overall, the time-risk cost monotonically reduces with the increase in the number of speeds of GMDM. Note that T$^\star$ minimizes the joint travel time and risk costs, thus the individual costs could change in a non-monotonic manner.} It is clear that GMDM provides enhanced maneuverability with multiple speeds while substantially reducing the risks, thereby providing superior time-risk optimal paths. Furthermore, the use of multiple speeds and thus turning radii gives smoother paths that can better adapt to the shape of the obstacles. 

\section{ Practical Considerations}\label{sec:limits}
{

This section discusses the practical considerations for the implementation of GMDM in real applications. One important consideration is that GMDM is a first-order kinematic model. While GMDM allows for the use of multiple linear speeds and turning rates and thus can provide good approximations of paths for a large class of vehicles \cite{L06}, there are instantaneous discontinuities in the curvature (like in Dubins) and speed when switching from one motion primitive to the next. Although such a path can be reasonably tracked by vehicles like differential-drive robots \cite{Balkcom2002}, other vehicles like quadrotors might not be able to track these paths accurately or feasibly due to motor dynamics and high-order differential smoothness requirements \cite{Mellinger2011_MinSnap,Mellinger2012_MinSnap,Richter_2016MinSnap}. 

In such cases, GMDM should first be utilized to produce the global path quickly, after which smoothing or other techniques can be used to make the path tractable for the vehicle. A few options are available to produce dynamically feasible trajectories built upon GMDM paths. In applications where continuous curvature is required, the use of Fermat's spiral \cite{Lekkas2013_FermatSpiral} or clothoid arcs \cite{Fraichard2004_CCSteer} can be utilized; similarly, the linear speeds can be smoothed considering acceleration while maintaining the constraints on the curvature as described in \cite{Kucerov2020_MultiRadiiDubins}. However, these paths do not consider higher-order derivatives like jerk and snap. In this regard, GMDM-polynomial paths can be developed as an extension of the Dubins-polynomial paths presented in \cite{Bry2015_MinSnap}. This approach first plans the high-level path using GMDM and then optimally produces a dynamically feasible polynomial trajectory that satisfies the continuity constraints up to the specified derivative. Finally, deep and reinforcement learning-based approaches can be used to create policies that minimize the tracking error of a given GMDM path and the dynamically produced trajectory \cite{Almeida2019_RLMinSnap,Song2021_DroneRace,Hanover2024_TRODroneSurvey}.

Another consideration is that while GMDM is a suitable high-level motion model for a large class of  2D vehicles (e.g., vehicles with Ackermann steering geometry \cite{AckermannModel}, surface vessels \cite{SurfaceVehicle},  etc.), there are other Dubins-like vehicles that operate in 3D environments (e.g., autonomous underwater vehicles \cite{Zeng2015_AUVPlanning_Survey}, fixed-wing aircraft \cite{FixedWingAircraftEx}, etc.). One strategy to construct a simplified multi-speed model for 3D environments involves utilizing GMDM motions on a plane with an independent double integrator for the altitude kinematics. 
A similar approach for using Dubins was proposed by Karaman and Frazzoli in \cite{Karaman2010_KinodynamicRRT*}. Another approach is to derive and solve the kinematic equations extending GMDM in 3D environments, similar to the extension of the Dubins model described in \cite{Dubins3DModel}. 
These approaches to extend GMDM for 3D environments could be combined with the techniques described earlier in this section to ensure the construction of smooth 3D trajectories.

}

\vspace{0pt}
\section{Conclusions and Future Work}
\label{sec:MultispeedDubinsConclusion}
This paper developed a new motion model, called Generalized Multi-speed Dubins Motion Model (GMDM), that extends the Dubins model to incorporate multiple speeds for fast and safe maneuvering. GMDM allows each path segment to have any speed and turning rate. GMDM is mathematically proven to provide full reachability with closed form solutions. It is shown that GMDM solutions are suitable for real-time implementation. For constant speed, GMDM reduces to the original Dubins model. To the best of our knowledge, no existing model offers these capabilities. The effectiveness of GMDM was demonstrated for both time-optimal and time-risk optimal motion planning problems in various scenarios. For the time-optimal motion planning problem, GMDM generates paths with a solution quality approaching that of the optimal paths and gives a significant improvement over the Dubins model. 
In obstacle-rich environments, GMDM provides shorter paths while achieving a substantially lower risk. 

In addition to the extensions of GMDM described in Section~\ref{sec:limits}, future work includes using GMDM in advanced motion planners in real-world planning applications, e.g.,  dynamic environments~\cite{SMART_Shen2023}, human-robot interaction \cite{DARE_Yang2023} and dynamic target tracking \cite{POSE-R_Hare_2020}.  
{ GMDM could also be extended by considering non-constant speed on each segment to improve the solution quality of time-risk optimal planning.} Furthermore, GMDM can be extended to higher-order motion models (e.g., considering acceleration).

\vspace{-6pt}
\begin{appendix}
\subsection{Solution of the Forward Problem}\label{AppendixA}
\subsubsection{Proof of Proposition \ref{propcscforward}}\label{proofpropcscforward}
\begin{proof}
For a CSC path type, the forward equation is given as $\textbf{p}_f=C_{\textbf{u}_3,\tau_3}(S_{\textbf{u}_2,\tau_2}(C_{\textbf{u}_1,\tau_1}(\textbf{p}_0)))$, where
  \begin{subequations}
  \begin{align}
    \begin{split}\label{eq:CSCp1}
    \textbf{p}_1=C_{\textbf{u}_1,\tau_1}(\textbf{p}_0),
    \end{split}\\
    \begin{split}\label{eq:CSCp2}
    \textbf{p}_2=S_{\textbf{u}_2,\tau_2}(\textbf{p}_1),
    \end{split}\\
    \begin{split}\label{eq:CSCpf}
  \textbf{p}_f=C_{\textbf{u}_3,\tau_3}(\textbf{p}_2).
  \end{split}
    \end{align}
    \end{subequations}
  From (\ref{eq:CSCp1}), (\ref{eq:CMotionPrimitive}) and (\ref{eq:param1})  we get
\begin{subequations}
    \label{eq:CSCxytheta1}
    \begin{align}
    \begin{split}
    \label{eq:x1CSC} 
        x_1 ={}& x_0-r_{1}\big(\sin{\theta_0}-\sin{(\theta_0+\phi_1)}\big),
    \end{split}\\
    \begin{split}
    \label{eq:y1CSC}
        y_1 ={}&  y_0+r_{1}\big(\cos{\theta_0}-\cos{(\theta_0+\phi_1)}\big),
    \end{split}\\
    \begin{split}
    \label{eq:theta1CSC}
        \theta_1 ={}& \bmod(\theta_0+\phi_1,2\pi).
    \end{split}
    \end{align}
    \end{subequations}
From (\ref{eq:CSCp2}), (\ref{eq:CSCxytheta1}), (\ref{eq:SMotionPrimitive})  and (\ref{eq:param1}) we get
\begin{subequations} \label{eq:CSCxytheta2}
    \begin{align}
    \begin{split}
    \label{eq:x2CSC} 
        x_2 ={}& x_0-r_{1}\big(\sin{\theta_0}-\sin{(\theta_0+\phi_1)}\big) + \delta_2\cos{(\theta_0+\phi_1)},
    \end{split}\\
    \begin{split}
    \label{eq:y2CSC}
        y_2 ={}&  y_0+r_{1}\big(\cos{\theta_0}-\cos{(\theta_0+\phi_1)}\big) + \delta_2\sin{(\theta_0+\phi_1)},
    \end{split}\\
    \begin{split}
    \label{eq:theta2CSC}
         \theta_2 ={}& \bmod(\theta_0+\phi_1,2\pi). 
    \end{split}
    \end{align}
    \end{subequations}
From (\ref{eq:CSCpf}), (\ref{eq:CSCxytheta2}), (\ref{eq:CMotionPrimitive})  and (\ref{eq:param1}) we get
\begin{subequations}
    \label{eq:CSCxythetaf}
    \begin{align}
    \begin{split}
    \label{eq:xfCSC} 
        x_f ={}& x_0-r_{1}\sin{\theta_0}-r_{31}\sin{(\theta_0+\phi_1)} + \\ {}& \delta_2\cos{(\theta_0+\phi_1)}+r_{3}\sin{(\theta_0+\phi_1+\phi_3)},
    \end{split}\\
    \begin{split}
    \label{eq:yfCSC}
         y_f ={}&  y_0+r_{1}\cos\theta_0+ r_{31}\cos(\theta_0+\phi_1)+ \\
        & \delta_2\sin(\theta_0+\phi_1) -r_{3}\cos(\theta_0+\phi_1+\phi_3),
    \end{split}\\
    \begin{split}
    \label{eq:thetafCSC} 
         \theta_f ={}&  \bmod(\theta_0+\phi_1+\phi_3,2\pi).
    \end{split}
    \end{align}
    \end{subequations}
\end{proof}

\vspace{-6pt}
\subsubsection{Proof of Proposition \ref{propcccforward}} \label{proofpropcccforward}
\begin{proof}
For a CCC path type the forward equation is given as $\textbf{p}_f=C_{\textbf{u}_3,\tau_3}(C_{\textbf{u}_2,\tau_2}(C_{\textbf{u}_1,\tau_1}(\textbf{p}_0)))$ where
  \begin{subequations}
  \begin{align}
    \begin{split}\label{eq:CCCp1}
    \textbf{p}_1=C_{\textbf{u}_1,\tau_1}(\textbf{p}_0),
    \end{split}\\
    \begin{split}\label{eq:CCCp2}
    \textbf{p}_2=C_{\textbf{u}_2,\tau_2}(\textbf{p}_1),
    \end{split}\\
    \begin{split}\label{eq:CCCpf}
  \textbf{p}_f=C_{\textbf{u}_3,\tau_3}(\textbf{p}_2).
  \end{split}
    \end{align}
    \end{subequations}
  From (\ref{eq:CCCp1}), (\ref{eq:CMotionPrimitive}) and (\ref{eq:param1}) we get
\begin{subequations}
    \label{eq:CCCxytheta1}
    \begin{align}
    \begin{split}
    \label{eq:x1CCC} 
        x_1 ={}& x_0-r_{1}\big(\sin{\theta_0}-\sin{(\theta_0+\phi_1)}\big),
    \end{split}\\
    \begin{split}
    \label{eq:y1CCC}
        y_1 ={}&  y_0+r_{1}\big(\cos{\theta_0}-\cos{(\theta_0+\phi_1)}\big),
    \end{split}\\
    \begin{split}
    \label{eq:theta1CCC}
         \theta_1 ={}&  \bmod(\theta_0+\phi_1,2\pi). 
    \end{split}
    \end{align}
    \end{subequations}
From  (\ref{eq:CCCp2}), (\ref{eq:CCCxytheta1}),  (\ref{eq:CMotionPrimitive}) and (\ref{eq:param1}) we get
\begin{subequations}
    \label{eq:CCCxytheta2}
    \begin{align}
    \begin{split}
    \label{eq:x2CCC} 
        x_2 ={}& x_0-r_{1}\sin{\theta_0}+ r_{12}\sin{(\theta_0+\phi_1)}+ \\ & r_{2}\sin{(\theta_0+\phi_1+\phi_2)},
    \end{split}\\
    \begin{split}
    \label{eq:y2CCC}
        y_2 ={}&  y_0+r_{1}\cos{\theta_0}- r_{12}\cos{(\theta_0+\phi_1)}- \\ &  r_{2}\cos{(\theta_0+\phi_1+\phi_2)},
    \end{split}\\
    \begin{split}
    \label{eq:theta2CCC}
         \theta_2 ={}&  \bmod(\theta_0+\phi_1+\phi_2,2\pi).
    \end{split}
    \end{align}
    \end{subequations}
From  (\ref{eq:CCCpf}), (\ref{eq:CCCxytheta2}),  (\ref{eq:CMotionPrimitive}) and (\ref{eq:param1}) we get
\begin{subequations}
    \label{eq:CCCxythetaf}
    \begin{align}
    \begin{split}
    \label{eq:xfCCC} 
        x_f ={}& x_0-r_{1}\sin{\theta_0}+ r_{12}\sin{(\theta_0+\phi_1)} + \\ & r_{23}\sin{(\theta_0+\phi_1+\phi_2)}   +r_{3}\sin{(\theta_0+\phi_1+\phi_2+\phi_3)},
    \end{split}\\
    \begin{split}
    \label{eq:yfCCC}
        y_f ={}& y_0+r_{1}\cos{\theta_0}- r_{12}\cos{(\theta_0+\phi_1)}- \\ &  r_{23}\cos{(\theta_0+\phi_1+\phi_2)} -r_{3}\cos{(\theta_0+\phi_1+\phi_2+\phi_3)},
    \end{split}\\
    \begin{split}
    \label{eq:thetafCCC} 
         \theta_f ={}& \bmod(\theta_0+\phi_1+\phi_2+\phi_3,2\pi). 
    \end{split}
    \end{align}
    \end{subequations}
\end{proof}

\vspace{-18pt}
\subsection{Solution of the Inverse Problem}\label{AppendixB}
\subsubsection{Proof of Proposition \ref{propcscinverse}} \label{proofpropcscinverse}
\begin{proof}
Using (\ref{eq:AB}) and (\ref{eq:theta1CSC}), we rewrite (\ref{propxfcsc}) and (\ref{propyfcsc}) as
    \begin{subequations}
     \label{eq:CSC_mod}
     \begin{align}
        \label{eq:CSC_mod1}
         -r_{31}\sin{\theta_1}+\delta_2\cos{\theta_1}={}& a, \\
         \label{eq:CSC_mod2}
        r_{31}\cos{\theta_1} + \delta_2\sin{\theta_1} ={}& b. 
     \end{align}
    \end{subequations}

Taking the squares of (\ref{eq:CSC_mod1}) and (\ref{eq:CSC_mod2}), adding them and rearranging, we get
        \begin{equation} \label{eq:CSCsolndelta2}
        \delta_2 = \sqrt{a^2+b^2-r^2_{31}}.
    \end{equation}

    Next, (\ref{eq:CSC_mod1})$\cdot$$\sin{\theta_1}$ - (\ref{eq:CSC_mod2})$\cdot$$\cos{\theta_1}$ gives 
   \begin{equation}
        a\sin{\theta_1}- b\cos{\theta_1}=-r_{31}. 
   \end{equation}

Then applying the following trigonometric identity 
\begin{equation}
\label{eq:linearsinusoids}
A\sin{\alpha}+B\cos\alpha=\sqrt{A^2+B^2}\sin(\alpha+\atantwo(B,A)),
\end{equation}  
we get 
\begin{equation}
        \label{eq:CSCsolntheta1}
        \theta_1=\arcsin\Big(\frac{-r_{31}}{\sqrt{a^2+b^2}}\Big)-\atantwo(-b,a).
\end{equation}

Since {$\phi_1=\omega_1\tau_1=\theta_{10}$}, 
we can solve for $\tau_1$ using (\ref{eq:CSCsolntheta1}). Similarly, since $\delta_2=v_2\tau_2$, we can solve for $\tau_2$ using (\ref{eq:CSCsolndelta2}). Finally, since  $\phi_3=\omega_3\tau_3=\theta_{31}$, 
we can solve for $\tau_3$ using (\ref{eq:CSCsolntheta1}). The solutions for $\tau_1$, $\tau_2$, and $\tau_3$ are given as
    \begin{equation}
     \tau_1= \frac{\theta_{10}}{\omega_1},
         \tau_2= \frac{\delta_2}{v_2}, \ \textrm {and} \
          \tau_3= \frac{\theta_{31}}{\omega_3}.
    \end{equation}

\end{proof}

\vspace{-0pt}
\subsubsection{Proof of Proposition \ref{propcccinverse}} \label{proofpropcccinverse}
\begin{proof}
Using (\ref{eq:AB}), (\ref{eq:theta1CCC}) and (\ref{eq:theta2CCC}),  we rewrite (\ref{propxfccc}) and (\ref{propyfccc}) as 
    \begin{subequations}
     \label{eq:CCC_mod}
     \begin{align}
        \label{eq:CCC_mod1}
         a-r_{12}\sin{\theta_1}={}&r_{23}\sin\theta_2,\\
         \label{eq:CCC_mod2}
         b+r_{12}\cos{\theta_1}={}&-r_{23}\cos\theta_2.
     \end{align}
    \end{subequations}
Taking the squares of (\ref{eq:CCC_mod1}) and (\ref{eq:CCC_mod2}), adding them and rearranging we get

\begin{equation}
a\sin\theta_1 - b\cos\theta_1 = (a^2 + b^2 + r_{12}^2 - r_{23}^2)/(2r_{12}).
\end{equation}
Using (\ref{eq:linearsinusoids}) and the fact that $|\phi_2|\in(\pi,2\pi)$ \cite{L06}, we get 
    \begin{equation}
    \label{eq:CCCsolntheta1}
        \theta_1=\pi-\arcsin\Big(\frac{a^2+b^2+r_{12}^2-r_{23}^2}{2r_{12}\sqrt{a^2+b^2}}\Big)-\atantwo(-b,a).
    \end{equation}
Next, we rewrite (\ref{eq:CCC_mod1}) and (\ref{eq:CCC_mod2}) as
   \begin{subequations}
     \label{eq:CCC_modagain}
     \begin{align}
        \label{eq:CCC_mod1again}
         a-r_{23}\sin\theta_2={}&r_{12}\sin{\theta_1},\\
         \label{eq:CCC_mod2again}
         b+r_{23}\cos\theta_2={}&-r_{12}\cos{\theta_1}.
     \end{align}
    \end{subequations}
Taking the squares of (\ref{eq:CCC_mod1again}) and (\ref{eq:CCC_mod2again}), adding them and rearranging we get
\begin{equation}
a\sin\theta_2 - b\cos\theta_2 = (a^2 + b^2 + r_{23}^2 - r_{12}^2)/2r_{23}.
\end{equation}
Again using (\ref{eq:linearsinusoids}) and the fact that $|\phi_2|\in(\pi,2\pi)$ \cite{L06}, we get
 \begin{equation}
    \label{eq:CCCsolntheta2}
        \theta_2 =\pi-\arcsin\Big(\frac{a^2+b^2+r_{23}^2-r_{12}^2}{2r_{23}\sqrt{a^2+b^2}}\Big)-\atantwo(-b,a). 
    \end{equation}
Since {$\phi_1=\omega_1\tau_1=\theta_{10}$},
we can solve for $\tau_1$ using (\ref{eq:CCCsolntheta1}). Similarly, since {$\phi_2=\omega_2\tau_2=\theta_{21}$,} 
we can solve for $\tau_2$ using (\ref{eq:CCCsolntheta1}) and (\ref{eq:CCCsolntheta2}). Finally, since {$\phi_3=\omega_3\tau_3=\theta_{32}$}, 
we can solve for $\tau_3$ using (\ref{eq:CCCsolntheta2}). The solutions for $\tau_1$, $\tau_2$, and $\tau_3$ are given as
    \begin{equation}
         \tau_1= \frac{\theta_{10}} {\omega_1}, 
         \tau_2=\frac{\theta_{21}}{\omega_2}, 
         \tau_3= \frac{\theta_{32}}{\omega_3}.
     \end{equation}

\end{proof}

\vspace{-12pt}
\subsection{Reachability Proofs} \label{AppendixC}
\subsubsection{Proof of Theorem \ref{Th:GMDM_CSC_reachability} }
\label{prf:GMDM_CSC_reachability}

\begin{proof} 
From the CSC solution in (\ref{eq:CSCsolns}), it is clear that the solutions for $\tau_i$ exist as long as the solutions for $\theta_1$ and $\delta_2$ exist. Note that $a$, $b$, and $r_{31}$ are real by their definitions.  The solution for $\theta_1$ exists if the domain of the $\arcsin$ term in (\ref{eq:CSCsolntheta1prop}) is between -1 and 1, i.e.,
\begin{equation}
        \label{eq:MultiSpeed_CSCL_Reach_Step}
        -1\leq \frac{-r_{31}}{\sqrt{a^2+b^2}} \leq 1.    
\end{equation}
This implies that $a^2 + b^2 \geq r_{31}^2$. This condition also makes sure that $\delta_2$ is real in (\ref{eq:CSCsolndelta2prop}).  By substituting $a=x_f-c$ and $b=y_f-d$ in the above condition we get $(x_f-c)^2 + (y_f-d)^2  \geq  r_{31}^2$. 
\end{proof}

\subsubsection{Proof of Theorem \ref{Th:GMDM_CCC_reachability}}
\label{prf:GMDM_CCC_reachability}

\begin{proof}
From the CCC solution in (\ref{eq:CCCsolns}), it is clear that the solutions for $\tau_i$ exist as long as the solutions for $\theta_1$ and $\theta_2$ exist. Note that $a$, $b$, $r_{12}$, and $r_{23}$ are real by their definitions. The solutions for $\theta_1$ and $\theta_2$ exist if the domains of the $\arcsin$ terms in (\ref{eq:CCCsolntheta1prop}) and (\ref{eq:CCCsolntheta2prop}) are between -1 and 1, i.e.,
\begin{subequations}
    \label{eq:Multispeed_CCC_Reach_Step}
    \begin{align}
    \label{eq:Multispeed_CCC_Reach_Step_a}
        -1 \leq \frac{a^2+b^2+r_{12}^2-r_{23}^2}{2r_{12}\sqrt{a^2+b^2}} \leq 1, \\
        \label{eq:Multispeed_CCC_Reach_Step_b}
        -1 \leq \frac{a^2+b^2-r_{12}^2+r_{23}^2}{2r_{23}\sqrt{a^2+b^2}} \leq 1.
    \end{align}
\end{subequations}
The above inequalities (\ref{eq:Multispeed_CCC_Reach_Step_a}) and (\ref{eq:Multispeed_CCC_Reach_Step_b}) imply that
\begin{subequations}
    \label{eq:Multispeed_CCC_Reach_Step2}
    \begin{align}
    \label{eq:Multispeed_CCC_Reach_Step_a2}
        \left(\frac{a^2+b^2+r_{12}^2-r_{23}^2}{2r_{12}\sqrt{a^2+b^2}}\right)^2 \leq 1, \\
        \label{eq:Multispeed_CCC_Reach_Step_b2}
        \left(\frac{a^2+b^2-r_{12}^2+r_{23}^2}{2r_{23}\sqrt{a^2+b^2}}\right)^2 \leq 1.
    \end{align}
\end{subequations}
Thus,
\begin{subequations}
    \begin{align}
        \left({a^2+b^2+r_{12}^2-r_{23}^2}\right)^2 \leq 4r_{12}^2(a^2+b^2), \\
        \left({a^2+b^2-r_{12}^2+r_{23}^2}\right)^2 \leq 4r_{23}^2(a^2+b^2).
    \end{align}
\end{subequations}
Putting $q=a^2+b^2$ we get 
\begin{subequations}
    \label{eq:Multispeed_CCC_Reach_Step3}
    \begin{align}
    \label{eq:Multispeed_CCC_Reach_Step_a3}
        \left(q+(r_{12}^2-r_{23}^2)\right)^2 \leq 4r_{12}^2q, \\
        \label{eq:Multispeed_CCC_Reach_Step_b3}
        \left(q-(r_{12}^2-r_{23}^2)\right)^2 \leq 4r_{23}^2q.
    \end{align}
\end{subequations}
Rearranging we get 
\begin{subequations}
    \label{eq:Multispeed_CCC_Reach_Step4}
    \begin{align}
    \label{eq:Multispeed_CCC_Reach_Step_a4}
        q^2+(r_{12}^2-r_{23}^2)^2 + 2q(r_{12}^2-r_{23}^2)-4r_{12}^2q\leq 0, \\
        \label{eq:Multispeed_CCC_Reach_Step_b4}
        q^2+(r_{12}^2-r_{23}^2)^2 - 2q(r_{12}^2-r_{23}^2)-4r_{23}^2q\leq 0.
    \end{align}
\end{subequations}
Adding we get 
\begin{equation}
    \label{eq:Multispeed_CCC_Reach_Step5}
        q^2 - 2(r_{12}^2+r_{23}^2)q +(r_{12}^2-r_{23}^2)^2\leq 0.
\end{equation}
Factorizing we get 
\begin{equation}
    \label{eq:Multispeed_CCC_Reach_Step6}
        (q -  (r_{12}+r_{23})^2)(q-(r_{12}-r_{23})^2)\leq 0.
\end{equation}
For LRL path type, $r_1>0$, $r_2<0$ and $r_3>0$. Thus, $r_{12}\geq 0$ and $r_{23}\leq 0$. For RLR path type $r_1<0$, $r_2>0$ and $r_3<0$. Thus, $r_{12}\leq 0$ and $r_{23}\geq 0$. For both these path types   
\begin{equation}
    \label{eq:Multispeed_CCC_Reach_Step7}
        (r_{12}+r_{23})^2 \leq (r_{12}-r_{23})^2.
\end{equation}
Therefore, from (\ref{eq:Multispeed_CCC_Reach_Step6}) and (\ref{eq:Multispeed_CCC_Reach_Step7}) we get
\begin{equation}
    \label{eq:Multispeed_CCC_Reach_Step8}
    (r_{12}+r_{23})^2   \leq  q \leq (r_{12}-r_{23})^2.
\end{equation}
Simplifying and substituting for $q$ we get, 
\begin{equation}
     r_{31}^2   \leq  (x_f-c)^2 + (y_f-d)^2 \leq (r_{12}-r_{23})^2.
\end{equation}

\end{proof}

\vspace{-6pt}
\subsubsection{Proof of Theorem \ref{Th:GMDM_full_reachability}}
\label{prf:GMDM_full_reachability}
\begin{proof} 
Consider the LSL and RSR path types. It is sufficient to show that the unreachable regions of these two path types are disjoint. This implies that the union of their reachable regions covers the entire SE(2) space. Based on the CSC reachability condition in Theorem \ref{Th:GMDM_CSC_reachability}, the unreachable regions of LSL and RSR path types for any final heading $\theta_f$ are described by open circles. Let $q^{LSL}\triangleq(c^{LSL},d^{LSL})$ and $q^{RSR}\triangleq(c^{RSR},d^{RSR})$ be the centres, and $|r_{31}^{LSL}|$ and $|r_{31}^{RSR}|$ be the radii of the unreachable circular regions of LSL and RSR path types, respectively.  We need to show that the distance between the centers of these circles is larger than or equal to the sum of their radii.  Note that for any given control, $r_1^{R}=-r_1^{L}<0$ and $r_3^{R}=-r_3^{L}<0$.

The sum of radii of these circles is 
\begin{equation} 
\label{sumofradii}
 |r^{LSL}_{31}|+|r^{RSR}_{31}|= |r^{L}_{3}-r^{L}_{1}|+|r^{R}_3-r^{R}_1|= 2|r^{L}_{3}-r^{L}_{1}|=2|r^{LSL}_{31}|.
\end{equation}
The distance between the centres of these circles is given as
\begin{equation}
    \label{eq:GMDM_FullReachShow1}
   dist(q^{LSL},q^{RSR}) = \sqrt{(c^{LSL}-c^{RSR})^2+(d^{LSL}-d^{RSR})^2}. 
   \end{equation}
Then using (\ref{eq:CD}) we get, 
\begin{equation} 
\label{distcenters}
\begin{split}   
dist(q^{LSL},q^{RSR})  &  = 2\sqrt{(r^L_{1})^2+(r^L_{3})^2-2r^L_{1}r^L_{3}\cos(\theta_0-\theta_f)} \\
&  \geq 2\sqrt{(r_{1}^{L})^2+(r_{3}^{L})^2-2(r_{1}^{L})(r_{3}^{L})}\\
& = 2 |r_{3}^{L}-r_{1}^{L}|\\
& =  2|r_{31}^{LSL}|.  
    \end{split}
\end{equation}

Thus, from (\ref{sumofradii}) and (\ref{distcenters}) the unreachable open circular regions of LSL and RSR path types are non-overlapping.    
\end{proof}

\end{appendix}

\bibliographystyle{IEEEtran}
\bibliography{References}

\begin{IEEEbiography}[{\includegraphics[width=1in,height=1.25in,clip,keepaspectratio]{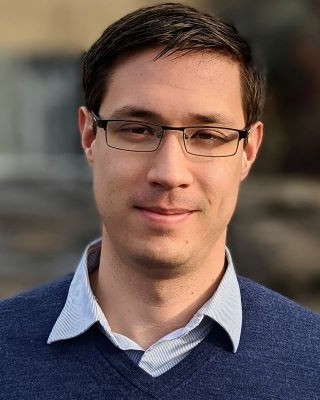}}]{James P. Wilson} received the B.S.E. degree in electrical engineering from the University of Connecticut, Storrs, CT, USA, in 2014, where he also received his Ph.D. degree in electrical engineering in 2023 with the Department of Electrical and Computer Engineering. He currently works as a Senior Research Engineer at RTX Technology Research Center, East Hartford, CT, USA. He is also an Adjunct Professor in the Department of Electrical and Computer Engineering at the University of Connecticut. His current research interests include data analysis and machine learning, robotics, motion planning, fault diagnosis and prognosis, and supervisory control in complex interconnected cyber-physical systems.
\end{IEEEbiography}

\begin{IEEEbiography}[{\includegraphics[width=1in,height=1.25in,clip,keepaspectratio]{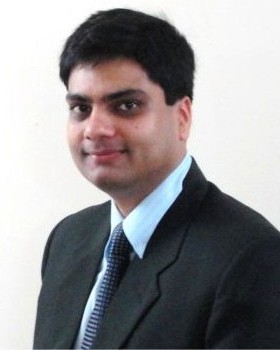}}]{Shalabh Gupta} received the B.E. degree in mechanical engineering from IIT Roorkee, India, in 2001. He received the M.S. degrees in mechanical and electrical engineering and the Ph.D. degree in mechanical engineering from the Pennsylvania State University, University Park, PA, USA, in 2004, 2005, and 2006, respectively. He is currently an Associate Professor at the Department of Electrical and Computer Engineering, University of Connecticut. His current research interests include distributed autonomy, cyber–physical systems, robotics, network intelligence, data analytics, information fusion, and fault diagnosis in complex systems. Dr. Gupta has published around 110 peer-reviewed journal and conference papers. He is a member of the IEEE, IEEE Systems, Man and Cybernetics Society, and ASME.
\end{IEEEbiography}

\begin{IEEEbiography}[{\includegraphics[width=1in,height=1.25in,clip,keepaspectratio]{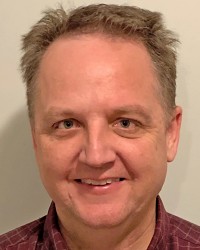}}]{Thomas A. Wettergren}  (Senior Member, IEEE) received the B.S. degree in electrical engineering and the Ph.D. degree in applied mathematics from Rensselaer Polytechnic Institute, Troy, NY, USA, in 1991 and 1995, respectively.
He works at the Naval Undersea Warfare Center, Newport, RI, USA, where he currently serves as the U.S. Navy Senior Technologist (ST) for Operational and Information Science. He also is an Adjunct Professor of Industrial
and Systems Engineering at the University of
Rhode Island, Kingston, RI, USA. He is the coauthor of the book titled Information-driven Planning and Control (MIT Press, 2021). His research interests include development of new analytical and computational methods for mathematical modeling, optimal planning, and
adaptive control of distributed groups. Dr. Wettergren is a member of the Society for Industrial and Applied
Mathematics (SIAM). He was a recipient of the NAVSEA Scientist of the Year, the Assistant Secretary of the Navy Top Scientists of the Year, and the IEEE-USA Harry Diamond awards.
\end{IEEEbiography}

\end{document}